\documentclass[lettersize,journal]{IEEEtran}
\usepackage[T1]{fontenc}
\usepackage{amsmath,amsfonts}
\usepackage{algorithmic}
\usepackage[ruled,linesnumbered]{algorithm2e}
\usepackage{array}
\usepackage{bm}
\usepackage[caption=false,font=normalsize,labelfont=sf,textfont=sf]{subfig}
\usepackage{textcomp}
\usepackage{stfloats}
\usepackage{url}
\usepackage{verbatim}
\usepackage{graphicx}
\usepackage{cite}
\usepackage{tabularx,booktabs}
\usepackage{multirow}
\usepackage{multicol}
\usepackage{colortbl} 
\usepackage{xcolor}
\usepackage{hyperref}
\usepackage{amssymb}
\usepackage{bbding}
\usepackage{pifont}
\usepackage{colortbl}

\usepackage{siunitx}

\definecolor{myblue}{RGB}{42, 73, 161}
\definecolor{myred}{RGB}{212, 45, 42}

\usepackage{hyperref} 

\hypersetup{
    colorlinks = true,   
    pdfborder = {0 0 0}, 
    linkcolor  = myred,
    citecolor  = myblue,
    urlcolor   = magenta
}

\newcommand{\e}{et al.}
\newcommand{\bb}{\boldsymbol}

\renewcommand{\epsilon}{\varepsilon}
\renewcommand{\phi}{\varphi}

\newtheorem{theorem}{ \bf Theorem}

\newtheorem{proposition}{\bf Proposition}

\newtheorem{remark}{\bf Remark} 

\newcolumntype{L}[1]{>{\raggedright\arraybackslash}p{#1}}
\newcolumntype{C}[1]{>{\centering\arraybackslash}p{#1}}
\newcolumntype{R}[1]{>{\raggedleft\arraybackslash}p{#1}}

\makeatletter
\newcommand{\removelatexerror}{\let\@latex@error\@gobble}
\renewcommand{\maketag@@@}[1]{\hbox{\m@th\normalsize\normalfont#1}}%
\makeatother
\usepackage{makeidx}
\makeindex

\begin{document}
    \title{Personalized Lower-limb Exoskeleton Assistance via Preference-based Bayesian Optimization}
	\author{Xiao-Yin Liu,~\IEEEmembership{} Guotao Li*, Weiqun Wang,~\IEEEmembership{} Zeng-Guang Hou*,~\IEEEmembership{Fellow, IEEE}
        \thanks{ This work was funded by the Noncommunicable Chronic Diseases National Science and Technology Major Project (2024ZD0528000, 2024ZD0528002), the National Natural Science Foundation of China under (Grant 62103412, Grant U22A2056, and Grant 62373013), and the Beijing Natural Science Foundation under (Grant L222053 and L232021). (*Corresponding authors: Guotao Li and Zeng-Guang Hou).}
		\thanks{Xiao-Yin Liu and Guotao Li are with the State Key Laboratory of Multimodal Artificial Intelligence Systems, Institute of Automation, Chinese Academy of Sciences, Beijing 100190, China, and also with the School of Artificial Intelligence, University of Chinese Academy of Sciences, Beijing 100049, China. (e-mail: guotao.li@ia.ac.cn).}

		\thanks{Zeng-Guang Hou is with the State Key Laboratory of Multimodal Artificial Intelligence Systems, Institute of Automation, Chinese Academy of Sciences, Beijing 100190, China, also with the School of Artificial Intelligence, University of Chinese Academy of Sciences, Beijing 100049, China, and also with CASIA-MUST Joint Laboratory of Intelligence Science and Technology, Institute of Systems Engineering, Macau University of Science and Technology, Macao, China. (e-mail: zengguang.hou@ia.ac.cn).}
	}
	
	\markboth{}%
	{Shell \MakeLowercase{\textit{et al.}}: A Sample Article Using IEEEtran.cls for IEEE Journals}
	\maketitle
	
	\begin{abstract}
	A significant challenge in exoskeleton robotics is the need to dynamically adapt control profiles to individual motion preferences, thereby ensuring both efficient and comfortable assistance. Currently, since user experience can serve as a comprehensive metric for evaluating the effectiveness of assistance, user preference-based optimization methods have been widely studied for parameter tuning. However, the existing methods rely heavily on extensive human-robot online interactions and suffer from slow optimization speed, which not only induces user fatigue but also compromises optimization effectiveness. Therefore, this paper aims to explore an efficient preference-based optimization framework for personalized exoskeleton assistance that can learn optimal parameters with minimal interaction. We propose a preference-based Bayesian optimization (\texttt{PbBO}) approach that can improve sample efficiency by leveraging knowledge about the sampling distribution of candidate sets. For optimizing six control parameters, \texttt{PbBO} can fast converge to user-preferred parameters with $90.7\%$ validation accuracy via $20$ iterations. Moreover, the hierarchical controller is designed to generate personalized torque for different tasks and achieve interaction torque tracking in real time. 
    The results of treadmill and outdoor experiments demonstrate that the optimized parameters can reduce metabolic rate by $14.5\% \thicksim 15.4\%$, heart rate by $6.3\% \thicksim 7.6\%$, and muscle activation by $6.7\% \thicksim 31.5\%$ compared to unassisted walking. Related Website: 
    \href{https://youtu.be/8X1SFqUU4G4}{PbBO}.

	\end{abstract}


	\section{Introduction} \label{sec:introduction}
\IEEEPARstart{M}{illions} of people suffer from mobility impairments caused by weakness of the human neuromotor system, often stemming from factors such as aging, muscle weakness, or stroke \cite{slade2022personalizing,enoka2016translating}. These individuals often walk more slowly and fatigue more easily. Exoskeletons that can augment locomotor strength and reduce energy losses have shown promise for addressing these impairments \cite{kim2019reducing,song2021optimizing}. However, in the real world, providing beneficial motion assistance is challenging due to individual differences in motion characteristics \cite{slade2022personalizing} and control parameters \cite{fang2021improving,lee2023user}.

A key goal in exoskeleton assistance is to design control strategies that meet the specific demands of individual users \cite{lee2023user,liu2025weight}. Human-in-the-loop optimization (HILO), which iteratively updates control parameters based on the user’s response, has become a promising approach to personalized assistance \cite{zhang2017human,2018science, review}. Current control strategies for HILO focus on updating parameters by optimizing physiological objectives (e.g., metabolic rate) \cite{2018science,2025hip,2025knee,2022TROHILO} or user experience objectives (e.g., user preference) \cite{slade2022personalizing,lee2023user,arens2025preference}. Because many factors, including comfort, balance, fatigue, and exertion, influence the user’s experience, a single physiological objective cannot adequately capture individual characteristics \cite{ingraham2022role}. The user’s preference that synthesizes a multi-factorial nature experience is a promising objective for tuning control parameters. 

However, the current optimization methods for users’ preference objectives still face serious challenges: 1) Given the difficulty of describing user preferences in precise mathematical language, accurately learning these preferences from user feedback emerges as the primary challenge; 2) Many optimization algorithms, such as Bayesian optimization, rely on a specific form of the objective function. Thus, how to integrate user preferences into the optimization framework to determine the optimal parameters becomes the second critical issue. Note that for personalized exoskeleton assistance, current studies learn a personalized torque curve within a single gait cycle by optimizing control parameters (peak time, rise time, fall time, and peak magnitude) \cite{zhang2017human,osti_10298403}. These parameters differ across individual users and tasks, and prolonged interaction can lead to user fatigue and an uncomfortable experience. Therefore, the solving efficiency of the optimization algorithm is crucial for personalized assistance with exoskeletons.

For \textit{challenge 1}, data-driven methods \cite{slade2022personalizing,lee2023user} and the Gaussian process (GP) model \cite{tucker2020preference,2021preference,arens2025preference} are applied to learn user preference from human feedback. The first method type requires collecting extensive offline data to train a preference model, which becomes impractical in complex scenarios \cite{slade2022personalizing}. The second type of method necessitates continuous online interaction between the user and the exoskeleton to learn the preference model, where prolonged interaction can lead to user fatigue \cite{2021preference}. For \textit{challenge 2}, one approach is to incorporate the learned preference model trained on offline preference data with the covariance matrix adaptation evolution strategy (CMA-ES) \cite{ingraham2022role,lee2023user}. Another type combines a preference model approximated by a GP with information gain optimization \cite{2021preference}. The third type directly utilizes a preference model to select the optimal parameters from the candidate set \cite{tucker2020preference,slade2022personalizing}. However, the above approaches require more online interaction time and cost when exploring higher-dimensional parameter spaces, which is challenging for subjects due to increased energy expenditure.

Accordingly, this paper proposes an efficient optimization framework that rapidly solves high-dimensional parameter optimization problems while minimizing individual energy expenditure through human-in-the-loop feedback. For preference learning, data-driven methods typically require large-scale offline human preference datasets, which are often impractical to collect due to high costs, especially for the rehabilitation scenario. To address this, we adopt an online interaction strategy for preference learning, leveraging a Gaussian Process (GP) model to refine user preferences during optimization iteratively \cite{chu2005preference}. 
In the context of HILO, Bayesian optimization (BO) has emerged as a widely adopted approach. While Kutulakos \e \cite{kutulakos2024simulating} demonstrated BO's superior performance compared to alternative methods (e.g., CMA-ES) for HILO applications, its requirement for a well-defined objective function remains a limitation. To address this, we integrate preference learning with Bayesian optimization, enabling simultaneous learning of the implicit objective function and optimization of control parameters.

\begin{figure*}
	\centering
	\includegraphics[width=0.99\textwidth]{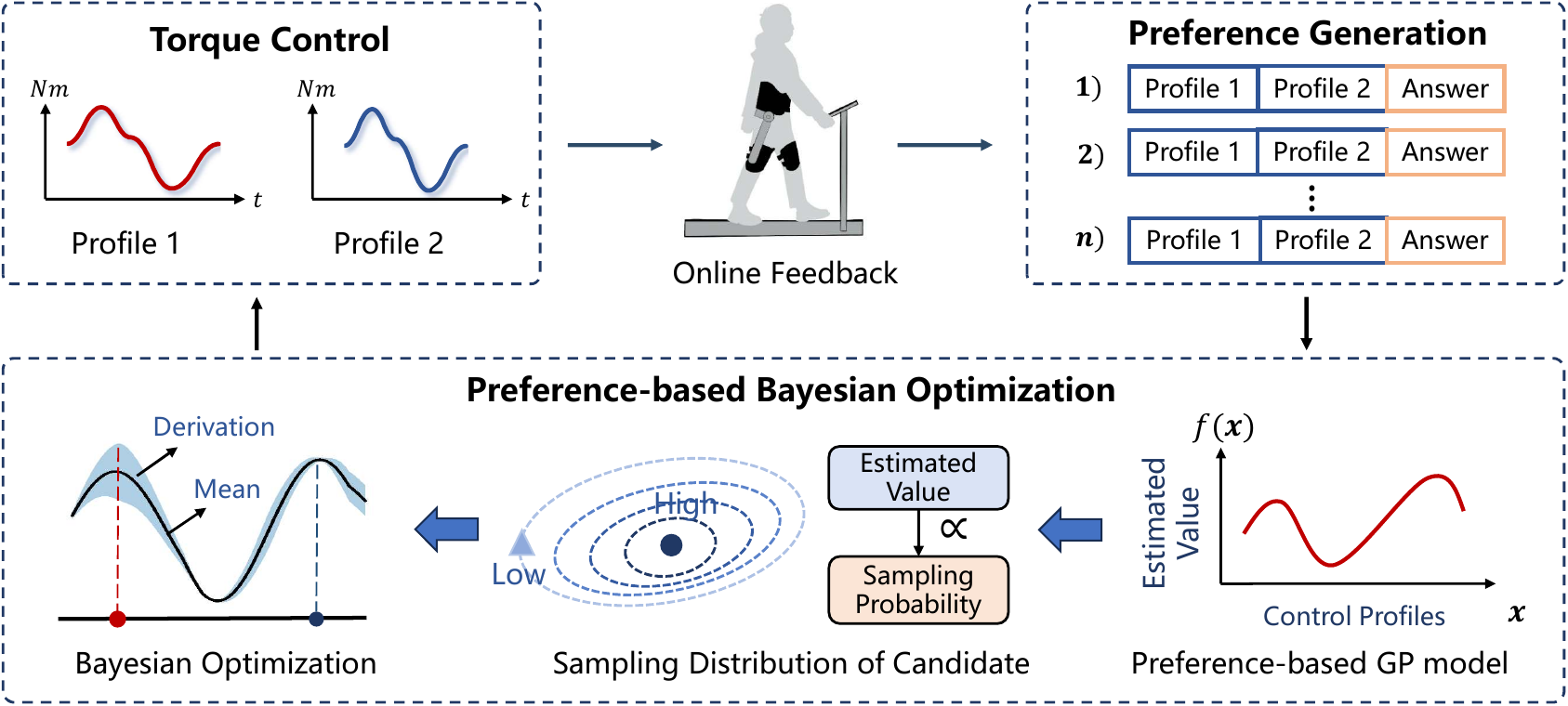}
	\caption{Overview of the proposed optimization framework. This method optimizes exoskeleton torque profiles based on user preference feedback. During each experimental trial, the user evaluates two distinct exoskeleton control parameter sets and provides comparative preference feedback. The preference data is utilized to iteratively refine a Gaussian process-based preference model, which subsequently generates sampling distributions for candidate sets. Bayesian optimization is then used to generate two offspring control-parameter sets for user evaluation. This interactive process continues iteratively until predefined termination criteria are satisfied.}. 
	\label{fig_framework}
\end{figure*}

Based on the above analysis, we propose a novel Preference-based Bayesian Optimization (\texttt{PbBO}) framework that integrates a Gaussian process-based preference model with BO. 
In this framework, to ensure optimization stability, we enhance the acquisition function in Bayesian optimization, effectively balancing exploration in the initial phase and exploitation in later stages. Moreover, to improve sample efficiency, we design an adaptive sampling distribution for candidate regions based on estimated preference values, prioritizing regions with higher preference values to accelerate convergence. 
Fig. \ref{fig_framework} shows the overall framework of preference-based Bayesian optimization for personalized exoskeleton assistance. The six control parameters of the hip torque curve are optimized continuously based on user feedback until the termination condition is reached. \texttt{PbBO} jointly updates both the preference model and control parameters through iterative user feedback, enabling efficient optimization without requiring explicit objective functions. 

To enable real-time, efficient exoskeleton assistance in outdoor environments, this paper presents a hierarchical control framework. The high-level layer is exploited to estimate the user’s walking velocity in real time and generate the corresponding personalized interaction torque curve based on \texttt{PbBO}. The low-level controller employs the nonlinear disturbance observer (NDOB) to estimate the interaction torque and controls motor torque to track the desired personalized torque \cite{mohammadi2013nonlinear,2017NDOB}.
To validate the performance of the proposed framework, this paper conducts experiments on both a treadmill and in an outdoor community setting. The experimental results indicate that \texttt{PbBO} can converge to the user-preferred parameters with $90.7\%$ accuracy across $20$ optimization trials, and the proposed assistance strategy yields a significant reduction of $14.5\% \thicksim 15.4\%$ in metabolic rate, $6.3\% \thicksim 7.6\%$ in heart rate, and $6.7\% \thicksim 31.5\%$ in muscle activation compared to unassisted walking. The key contributions can be summarized as: 
\begin{enumerate}
    \item A novel user preference optimization framework, named \texttt{PbBO}, is proposed for lower-limb exoskeleton assistance, which can accurately optimize preference parameters with minimal interaction.
    \item An adaptive sampling distribution for the candidate set and the trade-off coefficient are introduced to enhance both convergence efficiency and optimization stability of the proposed framework.
    \item A hierarchical control architecture is developed to precisely generate and track personalized assistance torque, which effectively reduces human energy expenditure for different walking tasks.
\end{enumerate}

The structure of this paper is organized as follows:
Section \ref{sec:Related work} reviews prior research in preference learning and human-in-the-loop optimization for personalized exoskeleton assistance.
Section \ref{model} presents the architecture of the hip exoskeleton system, the dynamic model, the relevant control parameters, and the hierarchical control strategy.
Section \ref{method} introduces key concepts in preference learning and Bayesian optimization, and then provides a detailed explanation of the proposed preference-based Bayesian optimization framework.
Section \ref{experiments_setup} describes the experimental protocols for both indoor treadmill and outdoor testing.
Section \ref{experiments} presents experimental results that validate the effectiveness of the proposed optimization approach.
Section \ref{Discussion} provides an in-depth discussion of personalized assistance and the study's limitations.
Finally, Section \ref{conclusion} concludes the paper.


	\section{Related Works}\label{sec:Related work}
This section reviews state-of-the-art research in user motion preference learning and human-in-the-loop optimization methodologies for exoskeleton assistance. Then, a comprehensive comparison is presented to highlight the key differences and advantages of the proposed optimization framework compared with existing techniques.

\subsection{User Motion Preference Learning}
User motion preference learning is the process of developing a computational model that discerns which set of motion trajectories best aligns with human preferences. The objective of preference learning is to guide the optimization direction of parameters by user feedback.
Current research in user preference learning can be broadly categorized into two main approaches: data-driven techniques \cite{slade2022personalizing,lee2023user} and statistical-based methods \cite{tucker2020preference,2021preference,arens2025preference}. 
For data-driven approaches, Slade \e \cite{slade2022personalizing} trained a classifier to estimate the probability of the first control parameter set outperforming the second, facilitating the ranking of multiple parameter configurations. Lee \e \cite{lee2023user} trained the RankNet model based on the Bradley-Terry framework to classify user motion preferences. The data-driven approaches require collecting preference data from multiple subjects, which incurs non-negligible data acquisition costs. 
For statistical-based approaches, Tucker \e \cite{tucker2020preference} and Arens \e \cite{arens2025preference} developed a preference learning model within a Gaussian process framework, which iteratively refines the model through finite online human-robot interactions until convergence criteria are satisfied. These methods rely solely on real-time human-exoskeleton interaction data, thereby eliminating the need for offline data collection. In certain application scenarios, such as rehabilitation settings, collecting large volumes of preference data incurs prohibitive costs and is practically infeasible. Accordingly, this work adopts the second approach for preference learning.

\subsection{Human-in-the-loop Optimization}
Human-in-the-loop optimization (HILO) is an iterative process that continuously refines control parameters through real-time human-robot interaction to maximize predefined optimization objectives. Existing research on exoskeleton assistance primarily focuses on two key aspects: optimization strategies and objectives. The optimization strategies are predominantly based on well-established and widely adopted algorithms, such as Bayesian optimization \cite{2018science,2025hip,2025knee,arens2025preference} and CMA-ES \cite{zhang2017human,ingraham2022role,lee2023user}. Regarding optimization objectives, they can be broadly categorized into two classes: explicit objective functions (e.g., physical metrics such as metabolic rate \cite{zhang2017human,2018science,2025knee,2022TROHILO}) and user preferences (implicit functions) \cite{slade2022personalizing,lee2023user,arens2025preference}. For the first type of objective, Zhang \e \cite{zhang2017human} employed the covariance matrix adaptation evolution strategy to minimize the metabolic rate measured through respirometry. Ding \e \cite{2018science} adopted a first-order dynamic model to estimate metabolic rate and subsequently optimized this objective using Bayesian optimization. Gordon \e \cite{2022TROHILO} evaluated metabolic rate online using a musculoskeletal model and solved the optimal assistance parameters via Bayesian optimization.


However, since multiple factors, including fatigue, balance, and comfort, substantially affect the user experience, relying solely on individual physical metrics fails to adequately characterize the efficacy of exoskeletal assistance. User preference has been incorporated as the second type of objective function in the optimization framework \cite{slade2022personalizing,lee2023user,arens2025preference}. Slade \e \cite{slade2022personalizing} developed a scoring model by integrating joint kinematic information to quantify user movement preferences, which was subsequently refined through online interaction to optimize control parameters. In a related approach, Lee \e \cite{lee2023user} trained a user preference model using offline preference data and combined it with Bayesian optimization (BO) to tune control parameters. While both methods demonstrate precise learning of user movement preferences, they rely heavily on extensive offline datasets, thereby incurring substantial interaction costs. Arens \e \cite{arens2025preference} developed a preference model using Gaussian process regression and enhanced optimization efficiency by incorporating domain knowledge on just-noticeable differences between assistance settings. This approach eliminates the need for offline dataset collection, requiring only limited online interactions to optimize the parameters.

\begin{table}
	\normalsize
	\caption{Comparison results of the proposed method with existing exoskeleton optimization methods.}
	\label{tab_comparsion}
	\centering
	\resizebox{9cm}{!}{
		\begin{tabular}{c|ccc|cc}
\toprule[1pt] 
\textbf{Objective} & \textbf{Related} & \textbf{Post} & \textbf{Outdoor} & \textbf{Parameter} & \textbf{Iteration}\\
\textbf{function}& \textbf{works}& \textbf{train} & \textbf{deploy} & \textbf{dimension}& \textbf{number}\\

\midrule
\multirow{3}{*}{\textbf{ Metabolic}}&Ref. \cite{zhang2017human}& \XSolidBrush & \XSolidBrush & 4 &N/A\\
&Ref. \cite{2018science}& \XSolidBrush &\XSolidBrush &2 &20 \\
&Ref. \cite{2022TROHILO}& \XSolidBrush & \XSolidBrush& 4&24 \\
\midrule
\multirow{4}{*}{\textbf{ Preference}}&Ref. \cite{slade2022personalizing}& \Checkmark &\Checkmark & 4& N/A \\
&Ref. \cite{lee2023user}& \Checkmark & \XSolidBrush & 4 & 50 \\
&Ref. \cite{arens2025preference}&  \XSolidBrush &  \XSolidBrush & 2 & 12 \\
&\texttt{\textbf{PbBO}}& \XSolidBrush &\Checkmark & 6 & 20 \\
\bottomrule[1pt]
\end{tabular}
}
\end{table}

\subsection{The Proposed Framework}
The proposed framework falls into the broad category of objective functions based on user preferences. Similar to prior works \cite{tucker2020preference,2021preference,arens2025preference}, this paper leverages online interaction data to optimize a Gaussian process-based preference model, eliminating the need for offline data collection. Subsequently, we integrate Bayesian optimization with the preference model to optimize control parameters and design a sampling distribution for candidate sets, thereby enhancing the algorithm's optimization efficiency. Table \ref{tab_comparsion} provides a systematic comparison between the proposed method and state-of-the-art approaches, highlighting critical differences in pretraining requirements, suitability for outdoor deployment, and computational efficiency metrics, such as the number of optimized parameters and the number of convergence iterations. In contrast to previous user preference-based optimization techniques, the proposed framework, named \texttt{PbBO}, demonstrates the following key differences and advantages:

\textbf{1) Optimization Efficiency:} Previous preference-based optimization methods \cite{slade2022personalizing,lee2023user,arens2025preference} are typically limited to parameter spaces of up to four dimensions and often require excessive iterations, which could lead to user fatigue. In contrast, \texttt{PbBO} integrates a preference model with Bayesian optimization, enabling efficient identification of near-optimal 6-dimensional preference parameters with minimal human-robot interactions. This approach significantly enhances both sampling and optimization efficiency; \textbf{2) Data Requirement:} Compared with the approaches in \cite{slade2022personalizing} and \cite{lee2023user}, \texttt{PbBO} eliminates the need for collecting preference data from multiple subjects. More importantly, it circumvents the generalization weakness inherent in data-driven preference network training across different subjects, thereby significantly reducing the data requirements; \textbf{3) Outdoor Deployment:} While previous studies \cite{lee2023user} and \cite{arens2025preference} have demonstrated parameter optimization capabilities in various subjects and tasks within laboratory settings, they lack experimental validation for outdoor applications with varying walking speeds. This paper introduces the hierarchical control framework engineered to address the dynamic requirements of outdoor ambulatory scenarios.

	\section{System Description and Control Framework}\label{model}
This section presents a description of the portable hip exoskeleton and its system architecture and derives the relationship between the motion-output torque and the human-robot interaction torque. Then, the critical control parameters that govern the torque profile are defined, and the hierarchical control architecture is described.

\subsection{System Description}
Fig. \ref{fig_composition} illustrates a schematic of the mechanical design and components of the hip exoskeleton. The system consists of a waist belt, adjustable elastic straps, thigh support frames, and a hip joint actuation system, with a total system weight of $3.8\text{ kg}$. The system employs a lightweight, autonomous bilateral hip exoskeleton designed to assist hip flexion/extension, delivering a maximum continuous torque of $9\text{ Nm}$ and a peak torque of $22\text{ Nm}$. Torque is generated by a quasi-direct-drive actuator (AK80-9 T-Motor, China). To ensure secure attachment, a waist orthotic and two thigh orthotics anchor the exoskeleton to the user’s pelvis and thighs, respectively. The two-degree-of-freedom (DOF) free-pivoting joint is applied to enable unrestricted hip movement.

\begin{figure}
	\centering
	\includegraphics[width=0.40\textwidth]{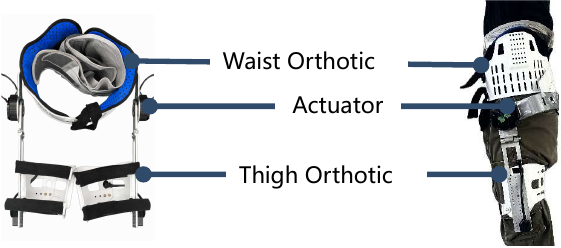}
	\caption{Untethered hip exoskeleton. The exoskeleton system comprises a waist-mounted battery pack, a drive motor, and a waist and thigh orthotic.}. 
	\label{fig_composition}
\end{figure}

The torque commanded $\bb{u}$ is converted to a desired motor current $\bb{i}$ using the following relationship, which incorporates the transmission ratio $N=9:1$ and the motor coefficient $K_t=0.89 \text{ Nm/A}$):
\begin{equation}
    \bb{i}=\frac{\bb{u}}{K_t\cdot N},
\end{equation}
Given a predefined torque profile $\bb{u}$, the motor can be directly controlled to provide walking assistance. The methodology for generating $\bb{u}$ is detailed in the following section.

\subsection{Dynamic Model and Control Parameters}
Since the motor's torque output does not directly reflect the torque experienced by the user, we develop a dynamic model of the exoskeleton to characterize the relationship between the motor-generated torque and the human-exoskeleton interaction torque. Furthermore, the key parameters requiring optimization in the hip joint torque profile are explicitly defined. 

\subsubsection{\textbf{Dynamic Model}} Let the rotational angle of robot joint be $\bb{q}\in \mathcal{R}^{n}$, and the output torque of motion be $\bb{u}\in \mathcal{R}^{n}$, where $n$ is the number of joints. Then, following \cite{han2023human}, the dynamic model of an exoskeleton robot, with consideration of human–robot interaction, can be given by
\begin{equation}\label{meq_1}
    \bb{M}(\bb{q})\ddot{\bb{q}}+\bb{C}({\bb{q}},\dot{\bb{q}})\dot{\bb{q}}+\bb{G}(\bb{q})+\bb{\tau}_{int}=\bb{u},
\end{equation}
where $\bb{M}(\bb{q})\in\mathcal{R}^{n \times n}$ and $\bb{C}({\bb{q}},\dot{\bb{q}})\in\mathcal{R}^{n \times n}$ are inertial matrix and velocity-dependent matrix respectively; $\bb{G}_q\in \mathcal{R}^{n}$ represent gravitational torque; $\bb{\tau}_{int}\in \mathcal{R}^{n}$ is the human-robot interaction torque. The dynamic model described by Eq. \eqref{meq_1} has the following properties \cite{2024chen}: 1) The matrix $\bb{M}(\bb{q})$ is symmetric, positive definite, and has bounded entries, i.e. $\|\bb{M}(\bb{q}\|\leq \varphi_1$; 2) The expression $\bb{M}(\bb{q})-2\bb{C}({\bb{q}},\dot{\bb{q}})$ exhibits skew-symmetry; The above parameters can be obtained from Lagrangian dynamics model. Dynamic components such as $\bb{M}$, $\bb{C}$, and $\bb{G}$ can be expressed as functions of mass, moment of inertia, and center of mass positions for each link of the hip exoskeleton \cite{kuccuktabak2024haptic}. These inertial properties can be derived from the exoskeleton's CAD model.

\subsubsection{\textbf{Control Parameters (Profile)}} The goal of the exoskeleton is to control the torque of the motor $\bb{u}$ to improve the comfort of the assistance and reduce the energy cost of the user. The human-robot interactive torque $\bb{\tau}_{int}$ generated by the motor's output torque $\bb{u}$ should align with the output torque curve of the human hip joint. Furthermore, the hip joint torque profile exhibits significant variability across individuals and assistance tasks (e.g., walking speed). Therefore, we propose a parameterized control framework defined by a set of adjustable parameters $\bb{x}$ to enable precise characterization of interaction torque profiles across different subjects and task conditions, and to enhance both comfort and assistance efficiency in human-robot interaction. In the following part, we present the detailed derivation of the key control parameters (torque profile).

\begin{figure}
	\centering
	\includegraphics[width=0.47\textwidth]{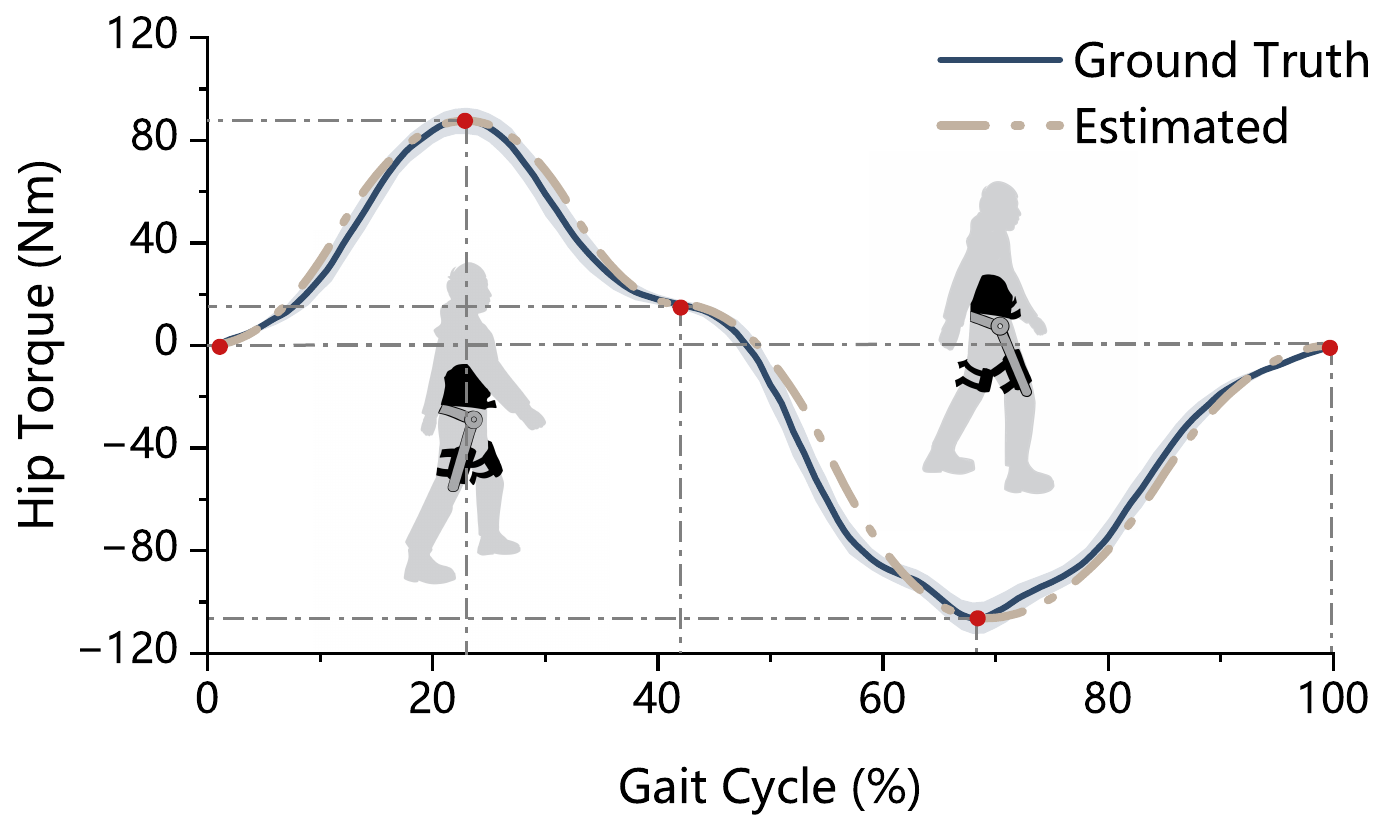}
	\caption{Comparison of the biological hip joint torque curve and the curve generated by Eq. \eqref{eq_m.3}. The solid line represents the actual torque profile over one gait cycle, while the dashed line represents the profile generated by the analytical model.}. 
	\label{fig_torque}
\end{figure}

By observing the average hip joint torque curves in the collected offline dataset \cite{camargo2021comprehensive}, we identified that the entire curve can be segmented into four distinct parts in each gait cycle (0\%-100\%), demarcated by three critical points (see Fig. \ref{fig_torque}). Therefore, the torque curve can be represented with five points $[(0,0),(a_1,b_1),(a_2,b_2),(a_3,b_3),(100,0)]$, that is, the control parameters $\bb{x}$ are $[a_1,b_1,a_2,b_2,a_3,b_3]$. Within each segment, the rate of torque variation initially increases and then decreases, with each subsection approximating a quadratic curve. Therefore, within a segment of the curve commencing at point $(x_0,y_0)$ and terminating at point $(x_1,y_1)$, we denote the inflection point where the curvature transforms as
\begin{equation}
    x_{s}=x_0+\eta*(x_1-x_0),\text{ } y_{s}=y_0+\eta*(y_1-y_0),
\end{equation}
where $\eta$ controls the position of the middle point between $(x_0,y_0)$ and $(x_1,y_1)$. Here, we set $\eta$ as $0.5$ in our experiments. Then, with the two endpoints as the vertices of the two parabolas, the curve expression of this section can be expressed as
\begin{equation}\label{eq_m.3}
    \begin{aligned}
        &L_0(x)=k_0(x-x_0)^2+y_0,\text{ }L_1(x)=k_1(x-x_1)^2+y_1\\
        &where \quad k_0=\frac{y_s-y_0}{(x_{s}-x_0)^2},\quad k_1=\frac{y_s-y_1}{(x_{s}-x_1)^2},
    \end{aligned}
\end{equation}

Here, we select the maximum value point $(a_1,b_1)$ and the minimum value point $(a_3,b_3)$ of the torque curve as two critical points, and another critical point $(a_2,b_2)$ is the approximate inflection point between $(a_1,b_1)$ and $(a_3,b_3)$. Note that the interaction torque typically represents approximately $15\%$ of the human biomechanical torque \cite{molinaro2024task}. Consequently, the normal range of human hip joint torque, typically $[-100,100]\text{ Nm}$, can be scaled to approximately $[-15,15]\text{ Nm}$ for human-robot interaction applications. Fig. \ref{fig_torque} compares the curve fitted from three key points and two start and end points with the torque curve in the public dataset \cite{camargo2021comprehensive}. This figure shows that Eq. \eqref{eq_m.3} better fits the true torque curve. Regarding the hip torques of the left and right joints, there exists a phase difference of $(a_3-a_1)\%$ between their gait cycles. This characteristic implies that only one set of control parameters is needed to generate the control curves for both joints.

\subsection{Hierarchical Control Architecture}
The above presents the method for generating personalized interaction torques $\bb{\tau}_{int}$ based on Eq. \eqref{eq_m.3}. Combined with the dynamic model in Eq. \eqref{meq_1}, the motor torques $\bb{u}$ can be computed. However, the interaction torque is difficult to measure directly. Herein, a nonlinear disturbance observer (NDOB) is employed to estimate the interaction torque $\bb{\tau}_{int}$, based on which the output motor torque $\bb{u}$ is subsequently determined.
To achieve adaptive outdoor assistance, it is necessary to continuously estimate the user's walking speed $v$ and subsequently generate the corresponding torque curve $\bb{\tau}_{int}$. Accordingly, a hierarchical control framework is adopted herein: the high-level estimates walking speed and produces personalized torque curves, while the low-level outputs motor torque via the dynamic model, thereby enabling efficient exoskeleton assistance. 

\subsubsection{\textbf{High-level Controller}} The human walking speed $v$ can be indirectly reflected by the duration of gait cycles $T_{cycle}$, and the two are approximately related as follows:
\begin{equation}\label{eq_24}
    v=A\cdot \exp (K\cdot T_{cycle})+B,
\end{equation}
where the parameters $A$, $B$ and $K$ can be derived by fitting the pairs of time $T_{cycle}$ and speed $v$ data collected during the experimental phase. The data pair is averaged across multiple subjects, and the same task is repeated multiple times to reduce the effect of random variation.
Here, the gait cycle duration $T_{cycle}$ can be calculated by the time it takes for the joint trajectory $\bb{q}$ to depart from and subsequently return to its reference angle $\bb{q}_{ref}$. Thereby, the corresponding walking speed $v$ can be estimated.

Based on the estimated speed $v$, the corresponding preference parameters can be selected, thereby generating the individualized interaction torque curve $\bb{\tau}_{int}$. However, since the curves generated from Eq. \eqref{eq_m.3} are defined over a fixed abscissa of $0\thicksim100$, the gait time $T_{tcycle}$ needs to be integrated to generate the final real-time torque.
In the actual control process, the estimated cycle time $T_{cycle}$ is used to determine the number of sampling points $n_{sample}=T_{cycle}/t_{sample}$, where $t_{sample}$ is the sampling time interval for the generated torque data. Thus, the generated curves can reflect the walking speed. Ultimately, personalized assistance can be achieved across varying walking speeds.

\subsubsection{\textbf{Low-level Controller}}  According to the Eq. \eqref{meq_1}, we can derive the following equation \cite{li2016adaptive,lizhi2022}:
\begin{equation}\label{eq_2_new}
\bb{W}_q(\bb{q},\bb{\dot{q}},\bb{\ddot{q}})\Psi_q+\bb{\tau}_{int}=\bb{u},
\end{equation}
where $\bb{W}_q(\bb{q},\bb{\dot{q}},\bb{\ddot{q}})\Psi_q=\bb{M}(\bb{q})\ddot{\bb{q}}+\bb{C}({\bb{q}},\dot{\bb{q}})\dot{\bb{q}}+\bb{G}(\bb{q})$, that is the dynamic model can be linear in a
set of physical parameters $\Psi_q$. 
Since the actual interaction torque $\bb{\tau}_{int}$ is difficult to obtain directly, the nonlinear disturbance observer (NDOB) \cite{mohammadi2013nonlinear,2017NDOB,huo2019force} is employed here to estimate it. The estimated interaction torque $\bb{\hat\tau_{int}}$ can be denoted as:
\begin{equation}
\left\{
\begin{aligned} 
        &\dot{\bb{Z}} = \bb{L}(\bb{q},\dot{\bb{q}})\left(-\bb{Z}-\bb{u}+\bb{Q}(\bb{q},\dot{\bb{q}})-\bb{P}(\bb{q},\dot{\bb{q}})\right),\\
        &\bb{\hat\tau_{int}} = -\bb{Z}  -\bb{P}(\bb{q},\dot{\bb{q}}),
    \end{aligned}
\right.
\end{equation}
where $\bb{Q}(\bb{q},\dot{\bb{q}}) = \bb{C}({\bb{q}},\dot{\bb{q}})\dot{\bb{q}}+\bb{G}(\bb{q})$, and the observer gain matrix $\bb{L}(\bb{q},\dot{\bb{q}}) \in \mathcal{R}^{n\times n}$ and the vector $\bb{P}(\bb{q},\dot{\bb{q}})\in \mathcal{R}^{n\times n}$ can be expressed as \cite{2017NDOB,2024chen}:
\begin{equation}\label{eq_28}
\left\{
\begin{aligned} 
        &\bb{L}(\bb{q},\dot{\bb{q}}) = \bb{X}^{-1} \bb{M}^{-1}(\bb{q}),\\
        &\bb{P}(\bb{q},\dot{\bb{q}}) = \bb{X}^{-1} \dot{\bb{q}},
    \end{aligned}
\right.
\end{equation}
where $\bb{X} \in \mathcal{R}^{n\times n}$ is a constant symmetric and invertible matrix. Here, we define the tracking error of interaction torque $\bb{e}$, denoted as $\bb{e}_r=\bb{\tau}_{int}^{d}-\bb{\tau}_{int}$. The estimation error can be denoted as $\Delta{\bb{\tau}}=\bb{\tau}_{int}-\bb{\hat\tau}_{int}$, and then the derivation of estimation error $\Delta\dot{\bb{\tau}}$ can be expressed as
\begin{equation}\label{eq_track}
    \Delta{\dot{\bb{\tau}}} =-\bb{X}^{-1} \bb{M}^{-1}(\bb{q}) \Delta{\bb{\tau}}+\dot{\bb{\tau}}_{int},
\end{equation}
The detailed derivation can be found in Appendix A. Then, according to \cite{mohammadi2013nonlinear}, the following proposition can illustrate the guarantee the convergence of $\Delta{\bb{\tau}}$:

\begin{proposition}[Theorem 1 of \cite{mohammadi2013nonlinear}]\label{pro_1}
    The interaction torque estimation error $\Delta{\bb{\tau}}$ converges to zero at a minimum exponential exponential convergence rate $\alpha = \lambda_{min}(\bb{\Gamma})/(2\varphi_1\|\bb{X}\|^2)$ when the following conditions are satisfied: 
    \begin{enumerate}
        \item The matrix $\bb{\Gamma} = \bb{X}+\bb{X}^T-\bb{X}^T\dot{\bb{M}}(\bb{q})\bb{X}$ is positive definite; 
        \item The rate of change of interaction torque is negligible relative to the estimation error in the dynamic Eq. \eqref{eq_track}, i.e., $\dot{\bb{\tau}}_{int}\approx 0$.
    \end{enumerate}
\end{proposition}
where $\lambda_{min}(\cdot)$ denotes the minimum eigenvalue of a matrix, $\|\cdot\|$ is the 2-norm of a vector or the induced 2-norm of a matrix, and $\varphi_1$ is the bound of $\bb{M}(\bb{q})$.

To achieve $\bb{\tau}_{int} \xrightarrow{} \bb{\tau}_{int}^{d}$, we use proportional-derivative (PD) control. The control law can be expressed as:
\begin{equation}\label{eq_control}
    \bb{u}= \bb{K}_p \hat{\bb{e}}_r + \bb{K}_d \dot{\hat{\bb{e}}}_r +\bb{W}_q\Psi_q+\bb{\tau}_{int}^d,
\end{equation}
where $\bb{K}_p \in \mathcal{R}^{n\times n}$ and $\bb{K}_d \in \mathcal{R}^{n\times n}$ are positive constant diagonal matrices, and $\hat{\bb{e}}_r = \bb{\tau}_{int}^{d}-\bb{\hat\tau}_{int}$. 
Given that $\bb{e}_r=\bb{\tau}_{int}^{d}-\bb{\tau}_{int}=\hat{\bb{e}}_r-\Delta\bb{\tau}$, and substituting Eq. \eqref{eq_control} into Eq. \eqref{eq_2_new}, we can conclude that:
\begin{equation}\label{eq_31}
    (\bb{K}_p+\bb{I}) {\bb{e}}_r + \bb{K}_d \dot{{\bb{e}}}_r = -\bb{K}_p\Delta\bb{\tau}-\bb{K}_d \Delta\bb{\dot\tau}.
\end{equation}
Furthermore, based on Proposition \ref{pro_1} (condition 2) and Eq. \eqref{eq_track}, Eq. \eqref{eq_31} can be written as
\begin{equation}\label{eq_32}
    \dot{{\bb{e}}}_r = \bb{A}{\bb{e}}_r + \bb{B} \Delta\bb{\tau},
\end{equation}
where $\bb{B} =\bb{K}_d^{-1}(-\bb{K}_p+\bb{K}_d\bb{X}^{-1} \bb{M}^{-1})$, and $\bb{A}=-\bb{K}_d^{-1}(\bb{I}+\bb{K}_p)$. Proposition \ref{pro_1} shows that the interaction torque estimation error $\Delta\bb{\tau}$ can converge to zero with exponential rate, i.e., $\|\Delta\bb{\tau}\|\leq C e^{-\alpha t}$. Here, $\alpha = \lambda_{min}(\bb{\Gamma})/(2\varphi_1\|\bb{X}\|^2)$ and $C$ is constant. Therefore, we can conclude from the theorem below that the tracking error in interaction torque converges.

\begin{theorem}[Exponential convergence]\label{the_2}
The proposed controller $\bb{u}$ in Eq. \eqref{eq_control} ensures the interaction torque tracking error $\bb{e}_r$ can converge to zero at a minimum exponential exponential convergence rate $\min \{\alpha,\lambda_{min}(\bb{K}_d^{-1}(\bb{I}+\bb{K}_p))\}$ under the system described in Eqs. \eqref{eq_2_new}-\eqref{eq_28}, when the conditions of Proposition \ref{pro_1} holds and $\bb{K}_p$ and $\bb{K}_d$ are positive constant diagonal matrices.
\end{theorem}

The proof can be found in Appendix B. 
This theoretically guarantees that the tracking error will ultimately converge to zero. The above hierarchical framework establishes the mapping from the personalized interaction torque $\bb{\tau}_{int}$ to the motor torque $\bb{u}$. In the following, we leverage this control framework to learn the personalized parameters via the proposed preference-based optimization algorithm, that is, elaborating on how personalized parameters $\bb{x}$ are learned via \texttt{PbBO}.
	\section{Efficient Preference-based Bayesian Optimization Method} \label{method}
This part first provides the basic notations for preference learning and Bayesian Optimization (BO) and summarizes the main problems we aim to solve. Then, we introduce a preference-based Gaussian Process (GP) regression model, following previous work \cite{chu2005preference,tucker2020preference,biyik2024active}. Then, we illustrate how to incorporate BO's optimization principles into preference learning and propose an efficient, fast preference-based Bayesian optimization algorithm. Finally, the theoretical analysis is provided to guarantee the effectiveness of the proposed optimization method.

\begin{figure}[!t]
\centering
\renewcommand{\algorithmicrequire}{\textbf{Input:}}
\renewcommand{\algorithmicensure}{\textbf{Output:}}
\removelatexerror
\begin{algorithm}[H]
\caption{Bayesian Optimization}
\label{alg1}

\begin{algorithmic}[1] 
\REQUIRE kernel function $\kappa(\bb{x},\bb{x}')$, offline buffer $\mathcal{D}=\emptyset$.
\ENSURE  the optimal parameters $\bb{x}^*=\bb{x}_{t_{iter}}$.
\FOR {$t=1,2,\cdots, t_{\text{iter}} $}
\STATE Select $\bb{x}_{t}=\arg\max \mu_{t-1}(\bb{x})+{\beta_{t-1}}\sigma_{t-1}(\bb{x})$.
\STATE Observe $y_t=f(\bb{x}_{t})+n_t, n_t\sim \mathcal{N}(0,\sigma_{n}^2)$, and add $(\bb{x}_t,y_t)$ to offline dataset $\mathcal{D}$.
\STATE Update mean $\mu_t$ and covariance $\sigma_t$ through $\mathcal{D}$.
\ENDFOR
\end{algorithmic}
\end{algorithm}
\end{figure}

\subsection{Problem Formulation}
In recent years, many studies have focused on the personalized, comfortable assistance provided by exoskeletons. They proposed different optimization algorithms to solve the optimal control parameters, that is 
\begin{equation}
\bb{x}^{\star} = \arg\max_{\bb{x}\in \chi} f(\bb{x}),
\end{equation}
where $\bb{x}\in \mathcal{R}^d$ is the control parameters that need to be optimized, $\chi$ is the space of control parameters $\bb{x}$, and $f:\chi \xrightarrow{}\mathbb{R}$ is the objective function, which can be defined through metabolic rate or assistance effectiveness. Bayesian optimization (BO), applied in HILO, has been widely used to solve the above optimization problem.

BO optimizes the control parameters by continuously collecting paired data. Algorithm \ref{alg1} shows the basic optimization steps of BO. BO selects the next control scheme $\bb{x}_t$ by maximizing the acquisition function based on the observed dataset. Then, the next objective value can be obtained through $f(\bb{x}_t)$. Furthermore, the posterior distribution can be updated based on the current observation. Through continuous iteration, the obtained value $\bb{x}_t$ gradually approaches the optimal value $\bb{x}^*$. However, the objective function $f$ is difficult to describe precisely in mathematical terms, such as the comfort of assistance and the preference for motion. Instead, the human preference can be easily obtained. Subjects are only required to choose the control scheme that best aligns with their motion preferences. Preference-based learning is an effective way to establish a connection between the objective function $f$ and the preference dataset $\mathcal{D}$.

The mathematical description of the preference dataset is as follows: Given a pair of control parameters $(\bb{x}_0,\bb{x}_1)$, humans choose which control scheme is preferred, i.e., $p\in\{-1,1\}$. The preference label $p=1$ indicates $\bb{x}_0\succ\bb{x}_1$ and $p = -1$ indicates $\bb{x}_1\succ\bb{x}_0$, where $\bb{x}_i\succ\bb{x}_j$ denotes that the control scheme $i$ is preferable to the control scheme $j$. Through constant human-robot interaction, the preference dataset $\mathcal{D}=\{\bb{x}_{i,0},\bb{x}_{i,1},p_i\}_{i=1}^m$ can be collected. 

This paper focuses on user preferences as the optimization objective to enhance the efficiency and comfort of exoskeleton assistance. The optimal control parameters $\bb{x}^*$ vary significantly across subjects and task conditions. The associated parameter optimization process imposes considerable time requirements, which may induce subject fatigue and limit practical implementation. BO is an efficient optimization method that can perform human-in-the-loop optimization tasks, but it requires a well-defined objective function. Therefore, this paper aims to address the following research challenge: how to effectively integrate preference learning with Bayesian optimization to develop an efficient preference-based optimization framework, thereby achieving rapid parameter optimization with minimal human-robot interactions.

\subsection{Preference-based Gaussian Process Model}
We aim to learn the latent objective function $f$ (latent human preference) based on the preference dataset $\mathcal{D}=\{\bb{x}_{i,0},\bb{x}_{i,1},p_i\}_{i=1}^m$, where $m$ denotes the number of user feedback and is the number of iterations in the optimization process. Then, let $\chi 
\in \mathcal{R}^d$ be the finite set of available control parameters. The cardinality of the finite set $\chi$ can be denoted as $|\chi|=N=2m$. For finite control schemes, the objective values can be written in the vector form: $\bb{f}=[f(\bb{x}_1),f(\bb{x}_2),...,f(\bb{x}_N)]^T$. Based on Bayesian theorem, the posterior probability of $\bb{f}$ can then be written as
\begin{equation}\label{eq_2}
    P(\bb{f}~|~\mathcal{D})\propto P(\bb{f}) P(\mathcal{D}~|~\bb{f}).
\end{equation}
where $P(\bb{f})$ is the prior probability and $P(\mathcal{D}|\bb{f})$ is the likelihood, which can be seen as the joint probability of observed data. The prior probability of these latent function values $\bb{f}$ can be viewed as a multivariate Gaussian, that is
\begin{equation}\label{eq_3}
  P(\bb{f}~|~ \bb{\mu}, \bb{\Sigma})=\frac{\exp{\left(-\frac{1}{2}\left(\bb{f}-\bb{\mu}\right)^T   \bb{\Sigma}^{-1}   \left(\bb{f}-\bb{\mu}\right)\right)}}{\left(2 \pi\right)^{N/2}\left|\bb{\Sigma}\right|^{1/2}},  
\end{equation}
where $\bb{\mu}\in \mathcal{R}^N$ and $\bb{\Sigma}\in \mathcal{R}^{N \times N}$ are the mean vector and the covariance matrix of the GP distribution for the $N$ items in control space. $\bb{\Sigma}$ is the $N\times N$ covariance matrix, where the $(i,j)$-th element is given by the covariance function $\kappa(\bb{x}_i,\bb{x}_j)$. Here, the modified radial basis function (RBF) kernel is used to approximate $\kappa(\bb{x}_i,\bb{x}_j)$ \cite{biyik2024active}: 
\begin{equation}\label{eq_4}
\begin{aligned}
       &\kappa(\bb{x}_i,\bb{x}_j) = \exp \left(-\theta\|\bb{x}_i-\bb{x}_j\|^2_2\right)- \bar{\kappa}(\bb{x}_i,\bb{x}_j)\\
       &\bar{\kappa}(\bb{x}_i,\bb{x}_j)=\exp\left(-\theta\|\bb{x}_i-\bb{x}_0\|^2_2-\theta\|\bb{x}_j-\bb{x}_0\|^2_2\right),
\end{aligned}
\end{equation}
where $\theta$ is a hyperparameter that controls the smoothness of sample paths, and $\bb{x}_0$ is the initial point. The modified term $\bar{\kappa}(\bb{x}_i,\bb{x}_j)$ is applied to measure the relative difference between two control parameters. Given preference dataset $\mathcal{D}$, we assume feedback may be disturbed by i.i.d. Gaussian noise: $y(\bb{x}_{t})=f(\bb{x}_{t})+n_t, n_t\sim \mathcal{N}(0,\sigma_{n}^2)$. Then, following the previous work \cite{chu2005preference}, we have
\begin{equation}\label{eq_5}
\begin{aligned}
        P\left(\bb{x}_0\succ \bb{x}_1~|~ \bb{f}\right)&=P\left(y(\bb{x}_{0})>y(\bb{x}_{1})~|~ f(\bb{x}_{0}),f(\bb{x}_{1})\right)\\
        &=\Phi\left[\frac{f(\bb{x}_{0})-f(\bb{x}_{1})}{\sqrt{2}\sigma_n}\right],
\end{aligned}
\end{equation}
where $\Phi$ is the cumulative distribution function of the standard normal. Thus, the joint probability of the preference dataset given the latent function $\bb{f}$ can be expressed as a product of the likelihood defined in Eq. \eqref{eq_5}, that is
\begin{equation}\label{eq_6}
    P\left(\mathcal{D}~|~ \bb{f}\right)=\prod_{i=1}^{m}\Phi\left[p_i\cdot\frac{f(\bb{x}_{i,0})-f(\bb{x}_{i,1})}{\sqrt{2}\sigma_n}\right].
\end{equation}

Then, substituting Eqs. \eqref{eq_3} and \eqref{eq_6} into Eq. \eqref{eq_2}, we can conclude the final model of preference-based GP. However, in Eq. \eqref{eq_2}, this posterior no longer follows a GP distribution and is difficult to analyze. Here, similar to \cite{tucker2020preference}, we approximate the posterior using the Laplace approximation, representing it as a multivariate Gaussian distribution, that is $P(\bb{f}|\mathcal{D})\sim \mathcal{N}(\bb{\mu}_{pos},\bb{\Sigma}_{pos})$. 
The core of the Laplace approximation is to approximate the posterior mean $\bb{\mu}_{pos}$ and posterior covariance $\bb{\Sigma}_{pos}$ using first- and second-order Taylor expansions. 
Then, the posterior mean $\bb{\mu}_{pos}$ can be estimated through maximizing the posterior distribution (maximum a posterior (MAP)), that is $\bb{\mu}_{pos}=\arg\max_{\bb{f}}P(\bb{f}|\mathcal{D})$, which is equivalent to minimize the following equation:
\begin{equation}\label{eq_7}
    \arg\min_{\bb{f}} ~L(\bb{f})=-\log P(\bb{f}~|~ \bb{\mu}, \bb{\Sigma}) - \log P\left(\mathcal{D}~|~ \bb{f}\right).
\end{equation}

Then, we assume that the prior mean of $\bb{f}$ is the zero function. Then, combining Eqs. \eqref{eq_3}, \eqref{eq_6} and \eqref{eq_7} together, the minimization of Eq. \eqref{eq_7} can be written as
\begin{equation}\label{eq_8}
    -\sum_{i=1}^{m}\log \Phi\left[p_i\cdot\frac{f(\bb{x}_{i,0})-f(\bb{x}_{i,1})}{\sqrt{2}\sigma_n}\right]+\frac{1}{2}\bb{f}^T \bb{\Sigma}^{-1}\bb{f}.
\end{equation}
The detailed optimization is provided in Appendix A. Then,
for posterior covariance $\bb{\Sigma}_{pos}$, Chu \e \cite{chu2005preference} derived the final form $\bb{\Sigma}_{pos} = (\bb{\Sigma}^{-1}+\mathbf{W})^{-1}$, where $\mathbf{W} \in \mathcal{R}^{N\times N}$ is the negative Hessian of the log-likelihood, where the $(i,j)$-th entry of $\mathbf{W}$ can be denoted as
\begin{equation}\label{eq_9}
    W_{i,j} = -\left.\frac{\partial^2 \log P\left(\mathcal{D}~|~ \bb{f}\right)}{\partial f(\bb{x}_i)\partial f(\bb{x}_j)}\right|_{\bb{f}=\bb{\mu}_{pos}}.
\end{equation}

At this point, we understand how to use the collected preference dataset $\mathcal{D}$ to approximate the posterior mean and covariance for the Laplace approximation, as described in Eq. \eqref{eq_2}. This allows us to establish and update the preference model (i.e., the latent function $\bb{f}$) based on the preference data. However, our ultimate goal is not merely to establish a preference model but to derive optimal preference parameters (control parameters) from it. Therefore, in the next part, we illustrate how to incorporate the preference model into Bayesian optimization to optimize these parameters.

\subsection{Preference-based Bayesian Optimization Algorithm}
In the above preference-based GP model, the hyperparameters $\theta$ and $\sigma_n$ should be determined before the optimization process. The matrices $\bb{\mu}$, $\bb{\Sigma}$, and $\mathbf{W}$ are updated at each iteration, after which the preference parameters are optimized. For simplicity, after $t$ iterations, we denote the preference dataset as $\mathcal{D}_t=\{\bb{x}_{i,0},\bb{x}_{i,1},p_i\}_{i=1}^t$, the value of preference data for latent function as $\bb{f}_t=\{f(x_{i,0}),f(x_{i,1})\}_{i=1}^t\in \mathcal{R}^{2t}$, and the corresponding mean, covariance and hessian matrices as $\bb{\mu}_t\in \mathcal{R}^{2t}$, $\bb{\Sigma}_t\in \mathcal{R}^{2t\times 2t}$ and $\mathbf{W}_t \in \mathcal{R}^{2t\times2t}$, respectively. Given the arbitrary preference data $\{\bb{x}_0,\bb{x}_1,p\}$, the $\bb{\mu}_t$, $\bb{\Sigma}_t$ and $\mathbf{W}_t$ can be updated through Eq. \eqref{eq_7}, Eq. \eqref{eq_4} and Eq. \eqref{eq_9}, respectively. 

However, the objective of this paper is to determine the optimal control parameter $\bb{x}$ using the latent function $f$. The key challenge is selecting the next data pair $(\bb{x}_{0},\bb{x}_{1})$ for preference collection to efficiently identify the optimal parameters $\bb{x}^*$. Here, we first consider how to select next value $\bb{x}_{t+1}$ based on the learned mean, covariance and hessian matrices $\bb{\mu}_t$, $\bb{\Sigma}_t$ and $\mathbf{W}_t$. We denote the value at arbitrary point $\bb{x}_{t+1}$ as $f_{t+1}=f(\bb{x}_{t+1})$. Then, the joint distribution of $\bb{f}_t$ and $f_{t+1}$ satisfies:
\begin{equation}\label{eq_10}
\left[\begin{array}{c}
\bb{f}_t\\
f_{t+1}
\end{array}\right]
=
\mathcal{N}\left(\mathbf{0},
\left[\begin{array}{cc}
\bb{\Sigma}_t&\bb{k}_t\\
\bb{k}_t^T & \kappa(\bb{x}_{t+1},\bb{x}_{t+1})
\end{array}\right]\right),
\end{equation}
where $\bb{k}_t=[\kappa(\bb{x}_{t+1},\bb{x}_{i,0}),\kappa(\bb{x}_{t+1},\bb{x}_{i,1})]_{i=1}^t\in \mathcal{R}^{2t}$ and $\kappa$ is defined in Eq. \eqref{eq_4}. The prediction distribution of $\bb{x}_{t+1}$ can be written as $P(f_{t+1}|\bb{f}_t, \bb{x}_{t+1})=\mathcal{N}(m_t(\bb{x}_{t+1}),c_t^2(\bb{x}_{t+1}))$, where mean $m_t$ and variance $c_t$ satisfy:
\begin{equation}\label{eq_11}
    \begin{aligned}
        &m_t(\bb{x}_{t+1}) =\bb{k}_t^T \left(\bb{\Sigma}_t+\sigma_n^2\bb{I}\right)^{-1}\bb{\mu}_t, \\
        &c_t^2(\bb{x}_{t+1})= \kappa(\bb{x}_{t+1},\bb{x}_{t+1})-\bb{k}_t^T\left(\bb{\Sigma}_t+\mathbf{W}_t^{-1}\right)^{-1}\bb{k}_t.
    \end{aligned}
\end{equation}

Then, applying the confidence bound criteria of Bayesian optimization, the next point $\bb{x}_{t+1}$ can be selected through the next equation:
\begin{equation}\label{eq_12}
    \bb{x}_{t+1}=\arg\max_{\bb{x}\in \chi} m_{t}(\bb{x})+{\beta_{t}}^{1/2}c_{t}(\bb{x}),
\end{equation} 
where $\beta_t>0$ is the trade-off coefficient that balances exploration and exploitation. The larger the value of $\beta_t$, the greater the exploration. Here, for the selection of the next pair data $(\bb{x}_{t+1,0},\bb{x}_{t+1,1})$, we set different $\beta_t$ values in Eq. \eqref{eq_12} to solve $\bb{x}_{t+1,0}$ and $\bb{x}_{t+1,1}$, respectively.
To solve for $\bb{x}_{t+1,0}$, an approximate $\beta_{t,0}$ value is used to strike a better balance between exploitation and exploration, thereby increasing the speed of finding a solution. For solving $\bb{x}_{t+1,1}$, we set a relatively smaller $\beta_{t,1}$ to fully exploit the observed dataset and ensure the stability of optimization. Therefore, the next preference data pair can be achieved by:
\begin{equation}\label{eq_13}
\begin{bmatrix}
\bb{x}_{t+1,0} \\
\bb{x}_{t+1,1}
\end{bmatrix}
=
\begin{bmatrix}
\arg\max_{\bb{x}_0\in \chi} \text{ } m_{t}(\bb{x}_0)+{\beta_{t,0}}^{1/2}c_{t}(\bb{x_0}) \\
\arg\max_{\bb{x}_1\in \chi} \text{ } m_{t}(\bb{x}_1)+{\beta_{t,1}}^{1/2}c_{t}(\bb{x_1})
\end{bmatrix}.
\end{equation}
Then, the user provides feedback $p_{t+1}$ by indicating which of the two control parameter sets better aligns with their movement preferences. The preference data $(\bb{x}_{t+1,0},\bb{x}_{t+1,1},p_{t+1})$ is added to the preference dataset $\mathcal{D}$ for updating parameters $\bb{\mu}_{t+1}, \bb{\Sigma}_{t+1}$ and $\mathbf{W}_{t+1}$.

\textbf{Determination for Trade-off Coefficient $\beta_t$:} Initially, due to the inaccuracy of the latent function $f$, predictions are less precise. At this stage, a larger $\beta_t$ is used to enhance exploration, thereby accelerating the optimization process. As the volume of preference data increases and the preference model becomes more accurate, the $\beta_t$ value should be appropriately reduced to ensure optimization stability. Therefore, $\beta_t$ is configured in a decaying form, that is
\begin{equation}\label{eq_decay}
    \beta_t=\gamma \beta_{t-1}=\gamma^t \beta_{0},
\end{equation}
where $\gamma\in(0,1)$ is the decay factor, and $\beta_0$ is the hyper-parameter can be adjust for different tasks. Here, for solving $(\bb{x}_{t,0},\bb{x}_{t,1})$, the initial $\beta_0$ has different values, denoted as $\beta_{0,0}$ and $\beta_{0,1}$, respectively. Algorithm \ref{alg2} shows the pseudo-code of the preference-based Bayesian optimization (\texttt{PbBO}) method. The sampling distribution $\omega$ used for data collection is important for optimization efficiency. The most current optimization methods apply a random sampling distribution to collect data. Although better performance can be achieved through a random distribution, optimization is often slow, and the computational cost is high. To further improve optimization efficiency, we introduce an adaptive sampling distribution $\omega$ that adjusts the sampling probability based on the estimated preference value. The details are given below.

\begin{figure}[!t]
\centering
\renewcommand{\algorithmicrequire}{\textbf{Input:}}
\renewcommand{\algorithmicensure}{\textbf{Output:}}
\removelatexerror
\begin{algorithm}[H]
\caption{Preference-based Bayesian optimization algorithm (\texttt{PbBO})}
\label{alg2}
\begin{algorithmic}[1] 
\REQUIRE kernel function $\kappa(\bb{x},\bb{x}')$, preference buffer $\mathcal{D}=\emptyset$.
\ENSURE  the optimal parameters $\bb{x}^*=\bb{x}_{t_{iter}}$.
\FOR {$t=1,2,\cdots, t_{\text{iter}} $}
\STATE Collect $n$ samples $\mathbf{D}\in \mathcal{R}^{n\times d}$ from parameter space $\chi\in \mathcal{R}^{d}$ according to sampling distribution $\omega_{t-1}$.
\STATE Select next preference data pair $[\bb{x}_{t,0},\bb{x}_{t,1}]$ from $\mathbf{D}$ through optimizing Eq. \eqref{eq_13}.
\STATE Run control schemes $[\bb{x}_{t,0},\bb{x}_{t,1}]$ and obtain preference $p_t$, and add $(\bb{x}_{t,0},\bb{x}_{t,1},p_{t})$ to preference dataset $\mathcal{D}$.
\STATE Update mean $\bb{\mu}_t$, covariance $\bb{\Sigma}_t$ and Hessian matrix $\mathbf{W}_t$ through Eq. \eqref{eq_7}, Eq. \eqref{eq_4} and Eq. \eqref{eq_9}.
\STATE Update the mean $m_t$ and variance $c_t$ of prediction distribution according to Eq. \eqref{eq_11}.
\ENDFOR
\end{algorithmic}
\end{algorithm}
\end{figure}

\textbf{Designation for Adaptive Sampling Distribution $\omega_t$:} The efficient convergence of the sampled space to the optimal solution's neighborhood is crucial for accelerating algorithmic convergence. Note that, at $t-$th iteration, the collected preference data is denoted as $\mathcal{D}_t=\{\bb{x}_{i,0},\bb{x}_{i,1},p_i\}_{i=1}^t$, and the corresponding value of latent function is $\bb{f}_t=\{f(\bb{x}_{i,0}),f(\bb{x}_{i,1})\}_{i=1}^t \in \mathcal{R}^{2t}$. Here, we use the Gaussian distribution to approximate the sampling distribution based on $\mathcal{D}_t$ and $\bb{f}_t$, that is $\omega_t \sim \mathcal{N}(\bb{\mu}_t^{s},\bb{\Sigma}_t^s)$. Since $f$ can reflect the quality of the current parameters, the larger the value, the better the parameters, so the sampling distribution should be shifted to the $\bb{x}$ corresponding to the larger $f$ value. The $\bb{f}_t$ value can be approximated by $\bb{\mu}_t$ optimized through Eq. \eqref{eq_7}. Therefore, the probability corresponding to each sample $\bb{x}$ can be expressed as
\begin{equation}
    p_i = \frac{\exp(\bb{\mu}_{t,i})}{\sum_{j=1}^{2t}\exp(\bb{\mu}_{t,j})},\quad i=1,2,...,2t.
\end{equation}
Furthermore, we can get the discrete probability distribution, that is $\mathbf{X}_t\sim \bb{P}_t$, where the random variable is $\mathbf{X}_t=[\bb{x}_{1,0},\bb{x}_{1,1},...,\bb{x}_{t,0},\bb{x}_{t,1}]$ and the corresponding probability is $\bb{P}_t=[p_1,p_2,...,p_{2t}]$. Then, we use maximum likelihood estimation to estimate the mean $\bb{\mu}_t^s$ and covariance $\bb{\Sigma}_t^s$ of the sampling distribution, which can be denoted as
\begin{equation}
\begin{aligned}
        \bb{\mu}_t^s &= \sum_{i=1}^{2t} \bb{P}_{t,i}\mathbf{X}_{t,i},\\
        \bb{\Sigma}_t^s &= \sum_{i=1}^{2t} \bb{P}_{t,i}\left(\mathbf{X}_{t,i}-\bb{\mu}_t^s\right)\left(\mathbf{X}_{t,i}-\bb{\mu}_t^s\right)^T,
\end{aligned}
\end{equation}

At each update iteration, $N_s$ points are sampled from the optimization space using the designed sampling distribution $\omega_t$. Moreover, to prevent the algorithm from falling into a local optimum, in addition to collecting samples from the above sampling distribution, a portion of the samples is randomly drawn from the sample space. The proportion of samples sampled from $\omega_t$ and randomly sampled is $f_s$. To ensure the safety of the next generation preference, we restrict the scope of the variable $\bb{x}$, that is $\texttt{clip}(\bb{x},\texttt{bounds})$, where \texttt{bounds} is the safe region set up artificially from the beginning. In summary, we have explained how to combine preference learning with Bayesian optimization and how to improve the stability and efficiency of the optimization algorithm.

\subsection{Theoretical Analysis}
This section analyzes the regret bound (convergence rate) of the proposed optimization algorithm from a theoretical perspective. Firstly, we introduce information gain in BO following \cite{GP2012}. At round $T$, given the observed vector $\bb{y}_T$ at the points \( S_T = \{\bb{x}_{1,0}, \bb{x}_{1,1}, \dots, \bb{x}_{T,0},\bb{x}_{T,1}\} \subset \chi \), the informativeness of the set $S_T$ about $f$ can be measured by information gain $I(\bb{y}_T;\bb{f}_T)$, where \( \bb{f}_{T} =\{f(\bb{x}_{i,0}), f(\bb{x}_{i,1})\}_{i=1}^T\), \( \bb{y}_{T}=\bb{f}_T+\bb{n}_T \), and $\bb{n}_T \sim \mathcal{N}(0,\sigma_n^2\bb{I})$. The information gain $I(\bb{y}_T;\bb{f}_T)$ is the mutual information between \( \bb{f}_{T}\) and \( \bb{y}_{T}\), which  can be expressed in terms of the predictive variances $c^2_{t-1}$ updated through Eq. \eqref{eq_11}:
\begin{equation}\label{eq_17}
    \frac{1}{2}\sum_{t=1}^T\log \left(1+\frac{c^2_{t-1}(\bb{x}_{t,0})}{\sigma_n^{2}}\right)\left(1+\frac{c^2_{t-1}(\bb{x}_{t,1})}{\sigma_n^{2}}\right).
\end{equation}

The $I(\bb{y}_T;\bb{f}_T)$ can quantify the reduction in uncertainty about $\bb{f}$ from revealing $\bb{y}$. The detailed derivation can be found in Appendix A. Then, the maximum information gain $\mathcal{H}_T$, which quantifies how many observation points $\bb{x}$ we need to fully describe the objective function $f$ during the optimization process, can be denoted as
\begin{equation}\label{eq_18}
    \mathcal{H}_T = \max_{S_T\subset \chi} I(\bb{y}_T;\bb{f}_T),
\end{equation}
where $\mathcal{H}_T$ measures the complexity of the reproducing kernel Hilbert space (RKHS), which describes the expressiveness of the objective function $f$ under a given kernel function $\kappa$. For the proposed optimization algorithm, we define the instantaneous regret $r_t$ for the preference data pair  $[\bb{x}_{t,0},\bb{x}_{t,1}]$ in round $t$,  which can be defined as  
\begin{equation}\label{eq_19}
    r_t=2f(\bb{x}^*)-f(\bb{x}_{t,0})-f(\bb{x}_{t,1}).
\end{equation}
Then, the cumulative regret $R_T$ after $T$ rounds is defined as the sum of regrets, that is, $R_T=\sum_{t=1}^T r_t$. The asymptotic property of an ideal optimization algorithm is no regret, that is $\lim_{T \to \infty} R_T/T=0$ \cite{srinivas2010gaussian}. Based on the information gain, we derive the below theorem for describing the regret bound of the proposed algorithm.

\begin{theorem}[Cumulative regret bound]\label{the_1}
For any $T$ and $\delta\in(0,1)$, the cumulative regret bound $R_T$ of Algorithm \ref{alg2} satisfies the below equation with probability at least $1-\delta$:
\begin{equation}
    R_T=\mathcal{O}\left(\sqrt{\beta_0(T)T\mathcal{H}_T}\right),
\end{equation}
\end{theorem}
where $\beta_0(T)$ is defined in Eq. (A.20). The proof can be found in Appendix B. This bound depends on the iteration round $T$, the maximum information gain $\mathcal{H}_T$, and the initial value $\beta_0(T)$. The convergence rate of the proposed algorithm is $\mathcal{O}(\sqrt{\beta_0(T)\mathcal{H}_T/T})$.

The maximum information gain, $\mathcal{H}_T$, is related to the kernel $\kappa(\cdot)$. Here, for radial basis function kernel, $\mathcal{H}_T=\mathcal{O}((\log T)^{d+1})$ \cite{srinivas2010gaussian}. Therefore, the cumulative regret bound satisfies $R_T=\mathcal{O}(\sqrt{\beta_0(T)T(\log T)^{d+1}})$. This theorem provides a regret bound for the proposed optimization based on preference learning. Compared with the regret bounds for standard GP-UCB algorithms \cite{GP2012,chowdhury2017kernelized}, the setting of $\beta_t$ differs significantly, and the bound is derived from the preference data pair.

\begin{remark}
    Note that the regret bound in Theorem \ref{the_1} holds for any $S_T$ since the maximum information gain $\mathcal{H}_T$ is achieved by maximizing the information gain $I(\bb{y}_T;\bb{f}_T)$ from $S_T\subset \chi$. For specific $S_T$, the regret bound $R_T$ of the proposed method is $\mathcal{O}(\sqrt{\beta_0(T)TI(\bb{y}_T;\bb{f}_T)})$. Here, we focus on the form of $I(\bb{y}_T|\bb{f}_T)$ defined in Eq. (A.10). The set $S_T$ is selected through Eq. \eqref{eq_13} and the sampling distribution $\omega$. As optimization progresses, the introduction of the sampling distribution $\omega$ guides samples toward the optimal region, resulting in higher sample similarity. Consequently, the elements of covariance $\bb{\Sigma}$ decrease, leading to smaller eigenvalues and, in turn, a reduction in $I$. Therefore, the actual regret bound becomes tighter after applying the sampling distribution, which can lead to a faster convergence rate.
\end{remark}
    \section{Experimental Setup}\label{experiments_setup}
This section provides experimental protocol details and the related evaluation metrics for exoskeleton assistance, and gives parameter configurations for algorithm optimization.

\begin{figure*}
	\centering
	\includegraphics[width=0.96\textwidth]{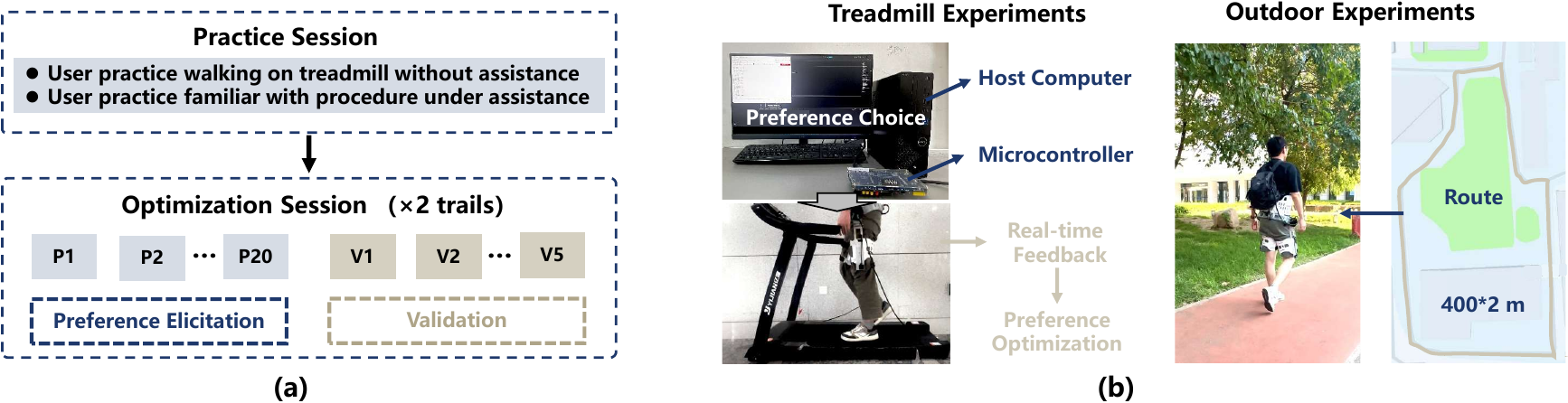}
	\caption{(a) shows the experimental protocol of preference optimization. In the first phase, users experience treadmill walking with and without exoskeleton assistance to familiarize themselves with the protocol. The second phase is an optimization stage, comprising a 20-step (P1-P20) \texttt{PbBO} process followed by a 5-step (V1-V5) validation procedure. (b) presents the scenario diagrams of the indoor treadmill experiment and the outdoor real-world environment.}. 
	\label{fig_procedure}
\end{figure*}
 
\begin{table}
	\normalsize
	\caption{Physical characteristic information of each subject.}
	\label{tab_subject}
	\centering
	\resizebox{7.2cm}{!}{
		\begin{tabular}{ccccc}
\toprule[1pt] 
Subject ID& Age & Gender &Height (m)& Mass (kg)\\

\midrule
AB01& $24$ & M &$1.83$& $77.6$\\
AB02& $27$ & M &$1.79$& $76.0$\\
AB03& $23$ & M &$1.71$& $72.0$\\
AB04& $26$ & F &$1.60$& $60.2$\\
AB05& $27$ & M &$1.77$& $74.5$\\
\bottomrule[1pt]
\end{tabular}
}
\end{table}

\subsection{Experimental Protocol} \label{dataset}
We recruited $5$ healthy, able-bodied participants, including $4$ males and $1$ female, with an average age of $25.4$ years, height of $1.74 \text{ m}$, and weight of $72.0 \text{ kg}$ (see Table \ref{tab_subject}). All participants provided written informed consent (approval number: IA-2502-020403) and are familiar with the hip exoskeleton from prior study protocols. Participants engage in overlapping experimental tasks, with $5$ individuals performing control parameter optimization on an indoor treadmill and 2-3 participants conducting validation trials in the treadmill and outdoor setting to assess exoskeleton performance under real-world conditions.

\subsubsection{\textbf{Preference Optimization Experiments}} The protocol comprises two phases similar to \cite{lee2023user,arens2025preference}: a practice session, and two optimization trials (See Fig. \ref{fig_procedure} (a)). The optimization framework is re-run for each trial. The practice session is designed to familiarize participants with treadmill walking while wearing the underpowered hip exoskeleton and to walk with powered torque assistance until they are familiar with the experimental procedure. For two optimization trials, each comprises preference elicitation and a validation session. During preference elicitation, participants iteratively selected between two control parameter settings (A/B) across $20$ generations. After the $20$th iteration, participants automatically proceed to a validation phase without explicit notification. This phase consisted of two additional comparisons, where the inferred optimal controller option is paired against randomly generated parameters.

Moreover, for preference selection, if a participant reported indifference, the algorithm processes this as two opposing preference votes (one for each option), ensuring the mean utility estimate remained unchanged while still reducing model uncertainty. This approach allows the acquisition function to update, typically leading to new option pairs in subsequent iterations. In the indoor treadmill experiments, each participant is required to complete three tasks at different walking speeds: $0.6, 1.0, 1.4\text{ m/s}$. Each task requires the above optimization trail to be followed. To mitigate fatigue effects on preference selection, participants receive $5$-minute breaks every $15$ minutes of walking and between sessions. The protocol limited the total daily experimental duration to under $2$ hours.

\subsubsection{\textbf{Validation Experiments}} Validation experiments are performed in both indoor treadmill and outdoor real-world settings. Fig. \ref{fig_procedure} (b) presents the indoor treadmill experimental setup, alongside a top-view schematic of the outdoor assisted walking route.
In outdoor experiments, participants are instructed to traverse the designated route at self-selected walking paces: slow, normal, and fast. 
To demonstrate the assistance effectiveness of \texttt{PbBO}, we compare the participants' metabolic rates, heart rate, and muscle activation under three conditions: without the exoskeleton (No Exo), with the exoskeleton but without assistance (Assist Off), and with the exoskeleton assisting (Assist On). The tested tasks include walking at different speeds on a treadmill, as well as walking along the aforementioned outdoor route.

\subsection{Evaluation Metrics and Parameters Configuration} \label{parameters}
The evaluation of \texttt{PbBO} consists of two aspects: the optimization effect of the algorithm itself and the assistance effect of the exoskeleton. Since the real hip torque profile is difficult to obtain directly, a Validation Accuracy (VA) is defined to indirectly measure the optimization effect, which is calculated as the ratio of the number of times the user selects the torque profile optimized by \texttt{PbBO} compared with the random profile to the total number of validation trials:
\begin{equation}\label{eq_26}
    \text{VA} = C_{opt}/C_{total},
\end{equation}
where $C_{opt}$ denotes the number of times the user preferred the optimized curve over the random curve during the testing phase, and $C_{total}$ represents the total number of choices.
The assistance performance of the exoskeleton is quantified by muscle activation $A_t$ measured by the Electromyography (EMG) device (ELONXI EMG-C4), metabolic rate $P_t$ measured using a COSMED K5 portable metabolic system, and heart rate $H_t$ measured by a GARMIN HRM device.
Muscle activation levels $A_t$ are typically quantified experimentally using the root mean square (RMS) of the EMG signal \(E_t\) (unit: \si{\micro\volt}), calculated as \cite{winters2012multiple,muscle}:
\begin{equation}
A_t=\sqrt{\frac{1}{T}\sum\nolimits_{t=1}^T E_t^2},
\end{equation}
with $T$ denoting the length of the time window employed. The heart rate $H_t$ (unit: BPM) is directly measured by the device through Bluetooth.
For the metabolic rate trials, where oxygen consumption $\dot{V}O_2$ and and carbon dioxide production $\dot{V}CO_2$ data are collected, the instantaneous metabolic cost $P_t$ (unit: W/kg) scaled by body mass $m$ is computed based on as \cite{zhang2017human,molinaro2024estimating}
\begin{equation}
    P_t=\frac{0.278\cdot \dot{V}O_2 +0.075\cdot \dot{V}CO_2}{m},
\end{equation}
The steady-state metabolic cost is determined by averaging the metabolic cost over the final three minutes of each five-minute trial. 
In preference-based Bayesian optimization, the hyperparameters $\theta$ in Eq. \eqref{eq_4}, noise $\sigma_n$, decay factor $\gamma$, initial value $\beta_0$ in Eq. \eqref{eq_decay}, the proportion of samples sampled from $\omega_t$, randomly sampled $f_s$, and total samples $N_s$ should be determined. For exoskeleton control, the dynamic model parameters $\bb{M}(\bb{q})$, $\bb{C}({\bb{q}},\dot{\bb{q}})$ and $\bb{G}(\bb{q})$ in Eq. \eqref{meq_1} are determined by exoskeleton model. The detailed parameter configuration is provided in Table \ref{Configuration}. The control parameters $\bb{X}=0.3$, $\bb{K}_p=0.5$ and $\bb{K}_d=0.005$ are adjusted to ensure the control stability.


\begin{table}
	\normalsize
	\caption{Parameters configuration during experiments.}
	\label{Configuration}
	\centering
	\resizebox{8.5cm}{!}{
        \begin{tabular}{ccc}
\toprule[1pt] 
\textbf{Parameter} & \textbf{Value} & \textbf{Description} \\
\midrule
$\theta$ & $0.1$ & Hyperparmater of RBF\\
$\sigma_n$ & $0.1$ & Noise level\\
$\gamma$ & $0.99$ & Decay factor of $\beta_t$\\
$\beta_0$ & $2.5$ & Initial value of $\beta_t$\\
$f_s$ & $0.9$ & Samples proportion from $\omega_t$\\
$\bb{M}(\bb{q})$ & $1.56\times 10^{-2}\bb{I}$ & Inertial matrix\\
$\bb{C}(\dot{\bb{q}},\bb{q})$ & $\bb{0}$ &Velocity-dependent matrix\\
$\bb{G}(\bb{q})$ & $0.879\sin (\bb{q})$ & Gravitational torque term\\
$\bb{X}$ & $0.3$ & Matrix of NDOB\\
\toprule[1pt] 
\end{tabular}
}
\end{table}


	\begin{figure*}
	\centering
	\includegraphics[width=1.0\textwidth]{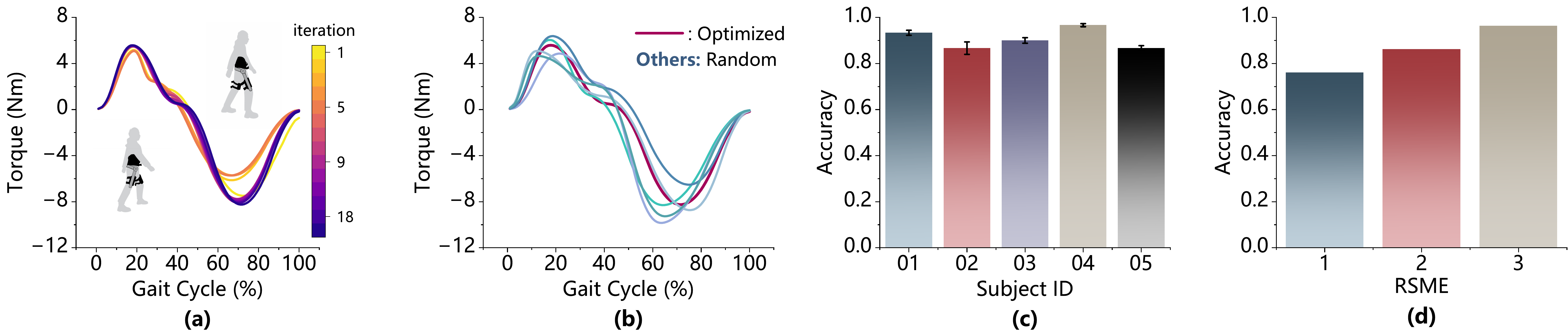}
	\caption{Evolution of the hip torque profile within the proposed optimization framework. (a) and (b) present the iteratively generated torque curves from the optimization session and randomly generated curves from the validation session, respectively. (c) demonstrates the validation accuracy of the optimized control parameters across different subjects. (d) displays the mean and standard deviation of validation accuracy categorized by root mean square error (RMSE) groups, where RMSE quantifies the deviation between the optimized and randomly generated torque curves.}. 
	\label{fig_validate_accuracy}
\end{figure*}

\section{Experimental Results}\label{experiments}
This section presents the experimental results of the aforementioned studies, which encompass the optimization efficiency and progression of the proposed framework (Section \ref{Progression}), the optimized control parameters across various walking speeds and subjects (Section \ref{speed}), the motor output torque curve and the assistance performance of optimized control parameters: reduction in metabolic cost, heart rate, and muscle activation in both indoor and outdoor experiments (Section \ref{Metabolic}).

\subsection{Optimization Progression of Hip Torque Profile}\label{Progression}
In Appendix C, this paper demonstrates that the proposed \texttt{PbBO} framework enhances optimization efficiency and stability by incorporating an adaptive sampling distribution $\omega_t$ and the trade-off factor $\beta_t$. To further evaluate the framework's performance, experiments are conducted with participants wearing powered hip exoskeletons under two distinct torque profile settings while walking on a treadmill at a constant speed of $0.6 \text{ m/s}$. The algorithm iteratively learns individual motion preferences and ultimately provides two optimized parameter recommendations. Fig. \ref{fig_validate_accuracy} (a) illustrates the progression of preferred hip torque profiles for a representative subject AB01 during both the optimization and validation sessions. The main session, which consists of optimization and validation phases, has an average duration of $20.6\pm4.6$ minutes per trial across all participants. Each trial includes two parts: preference elicitation and validation. Quantitative results demonstrate that the final optimized curve deviates substantially from the initial profile. As iterations proceed, the curve exhibits only marginal variations in the late phase, indicating that the optimization process converges to a steady state.
Fig. \ref{fig_validate_accuracy} (b) presents a comparative analysis between the randomly generated torque profiles and the optimized torque profile (obtained from the final iteration) during the validation phase. At a speed of $0.6 \text{ m/s}$, when making preference selections between the optimized profile and the random profile, Subject AB01 consistently chose the optimized one. This demonstrates a consistent preference for the optimized parameters over randomized alternatives.

\begin{figure*}
	\centering
	\includegraphics[width=0.99\textwidth]{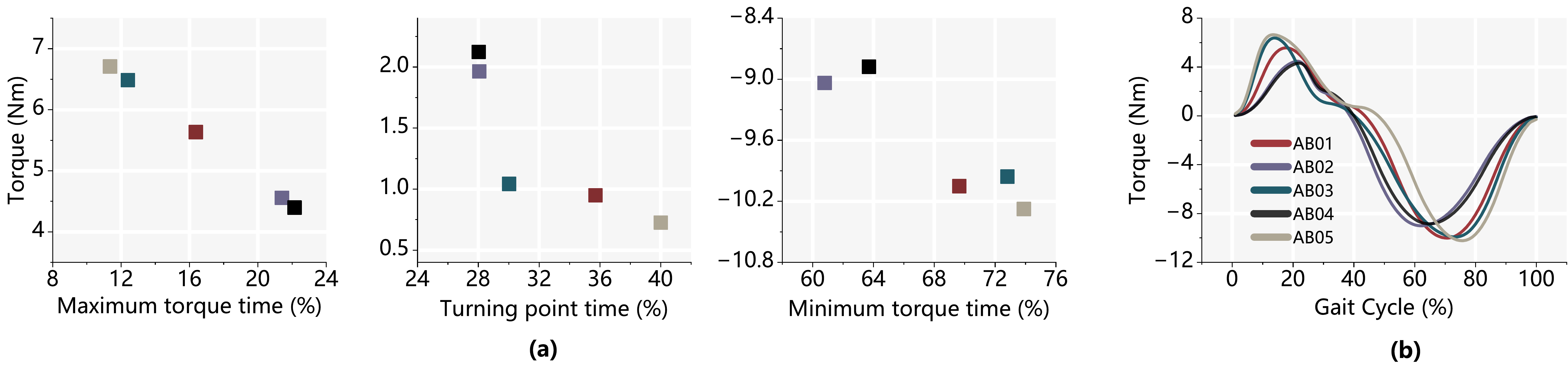}
	\caption{Optimization results of control parameters averaged across trials for different subjects under $1$ m/s. (a) illustrates the subject-specific optimized control parameters, while (b) compares the resulting hip torque profiles generated using these parameters across subjects.}. 
	\label{fig_nes}
\end{figure*}

Fig. \ref{fig_validate_accuracy} (c) presents the validation accuracy (VA) defined in Eq. \eqref{eq_26} for all participants (AB01-AB05). The validation accuracy is calculated as the percentage of trials in which the user selects the optimized torque profile over a randomized torque profile (Eq. \eqref{eq_26}), averaged across $5$ validation sessions. Note that each random parameter is confined within a prescribed range, and the resulting random curves are generally consistent with the overall trend of human motion.
The mean and standard deviation (SD) of the validation accuracy are derived from different tasks. The optimized torque profile selection rate reaches $90.7\pm1.3\%$ across all validation trials, averaged over all participants, and the maximum accuracy reaches $96.7\%$ for subject AB04, demonstrating a strong preference for the optimized parameters.  Fig. \ref{fig_validate_accuracy} (d) illustrates the correlation between the root-mean-square error (RMSE) of optimized torque profiles compared with random profiles and their corresponding selection accuracy rates during the validation phase. The results demonstrate an inverse relationship between RMSE and accuracy, such that lower estimation errors correspond to lower selection accuracy. This observed phenomenon provides insight into the underlying causes of misselection occurrences in the validation process.

\subsection{Optimization Results for Different Individuals and Speeds}\label{speed}
\subsubsection{\textbf{Optimization Results for Different Individuals}} Pronounced inter-individual differences in movement patterns across users give rise to distinct joint torque profiles. Fig. \ref{fig_nes} shows the optimized torque profile under the proposed optimization framework for different participants AB01-AB05. All participants walk on a treadmill at a constant speed of $1.0 \text{ m/s}$ to ensure consistent experimental conditions. Fig. \ref{fig_nes} (a) presents the distribution ranges of three critical points in the hip joint torque profiles for various users. The temporal and magnitude parameters for the maximum torque are distributed within $[5\%,25\%]$ of the gait cycle and $[2,8] \text{ Nm}$, respectively. For the minimum torque, the corresponding parameters range from $[50\%,80\%]$ of the gait cycle and $[-12,-8] \text{ Nm}$. The intermediate point parameters are distributed between $[25\%,45\%]$ of the gait cycle and $[0,3] \text{ Nm}$. As illustrated in Fig. \ref{fig_nes} (b), distinct optimization curves are observed across different users within the permissible parameter ranges. This result indicates that motion-preference characteristics vary to some extent across subjects.

\begin{figure}
	\centering
	\includegraphics[width=0.5\textwidth]{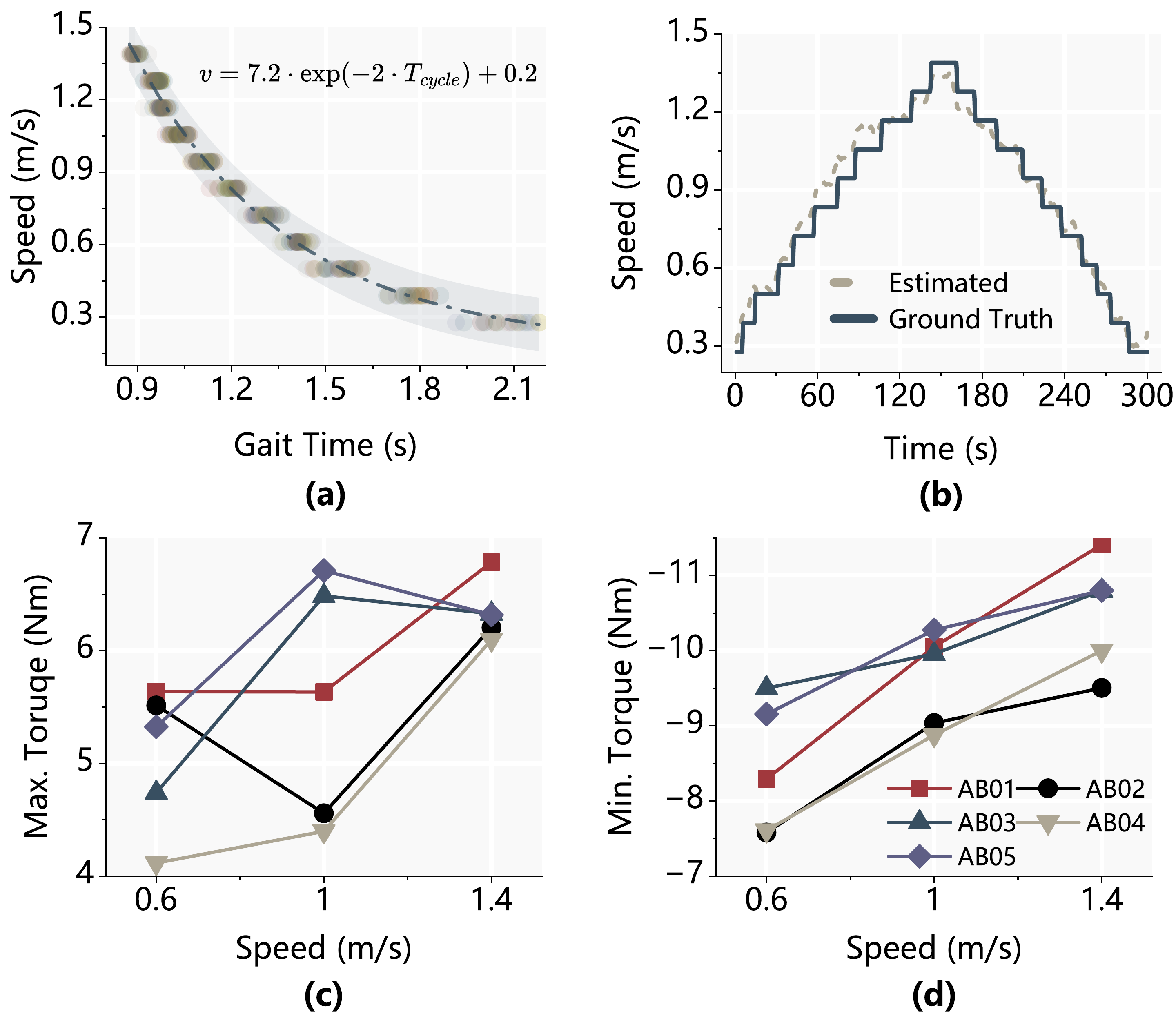}
	\caption{Velocity estimation results and optimization results for different walking speeds. (a) shows the relationship between walking speed $v$ and gait cycle time $T_{cycle}$. (b) compares the estimated speeds with the corresponding true values. (c) and (d) show the optimization results of control parameters for different tasks.}. 
	\label{fig_walking_speed}
\end{figure}

\subsubsection{\textbf{Optimization Results under Different Speeds}} In practical scenarios, the user's walking speed $v$ exhibits continuous variations. In this study, the duration of each gait cycle $T_{cycle}$ is measured using the motor's angular encoder, enabling regression-based estimation of the parameters in Eq. \eqref{eq_24} for walking speed prediction. Fig. \ref{fig_walking_speed} (a) illustrates the scatter distribution of encoder values across multiple subjects at varying walking speeds, with the regression analysis yielding parameter values of $A=7.2$, $B=0.2$, and $K=-2$. Building upon the aforementioned regression values, the exoskeleton can estimate the user's walking speed in real time during outdoor assistance. Fig. \ref{fig_walking_speed} (b) presents a comparative analysis between the estimated walking speed and the ground truth speed for participant AB01 on a treadmill. The RMSE between the estimated and actual walking speeds is calculated as $0.01 \text{ m/s}$. This verifies the acceptable accuracy of velocity estimation from motor angle measurements, which meets the demand for real-time user velocity estimation. 
Fig. \ref{fig_walking_speed} (c) and (d) present comparative diagrams of maximum torque and minimum torque for all subjects across three walking speeds, respectively. This further confirms that statistically significant inter-subject and inter-task variations are observed in subject-specific parameters. The results also indicate that as speed increases, the absolute values of both maximum and minimum torque exhibit an upward trend. 
This upward trend also validates the rationale of velocity-based optimization parameter selection in practical assistance scenarios.
By leveraging offline-optimized torque parameters for different speeds and incorporating real-time user speed estimation, adaptive assistance torque can be autonomously generated in practical scenarios.

\begin{figure}
	\centering
	\includegraphics[width=0.45\textwidth]{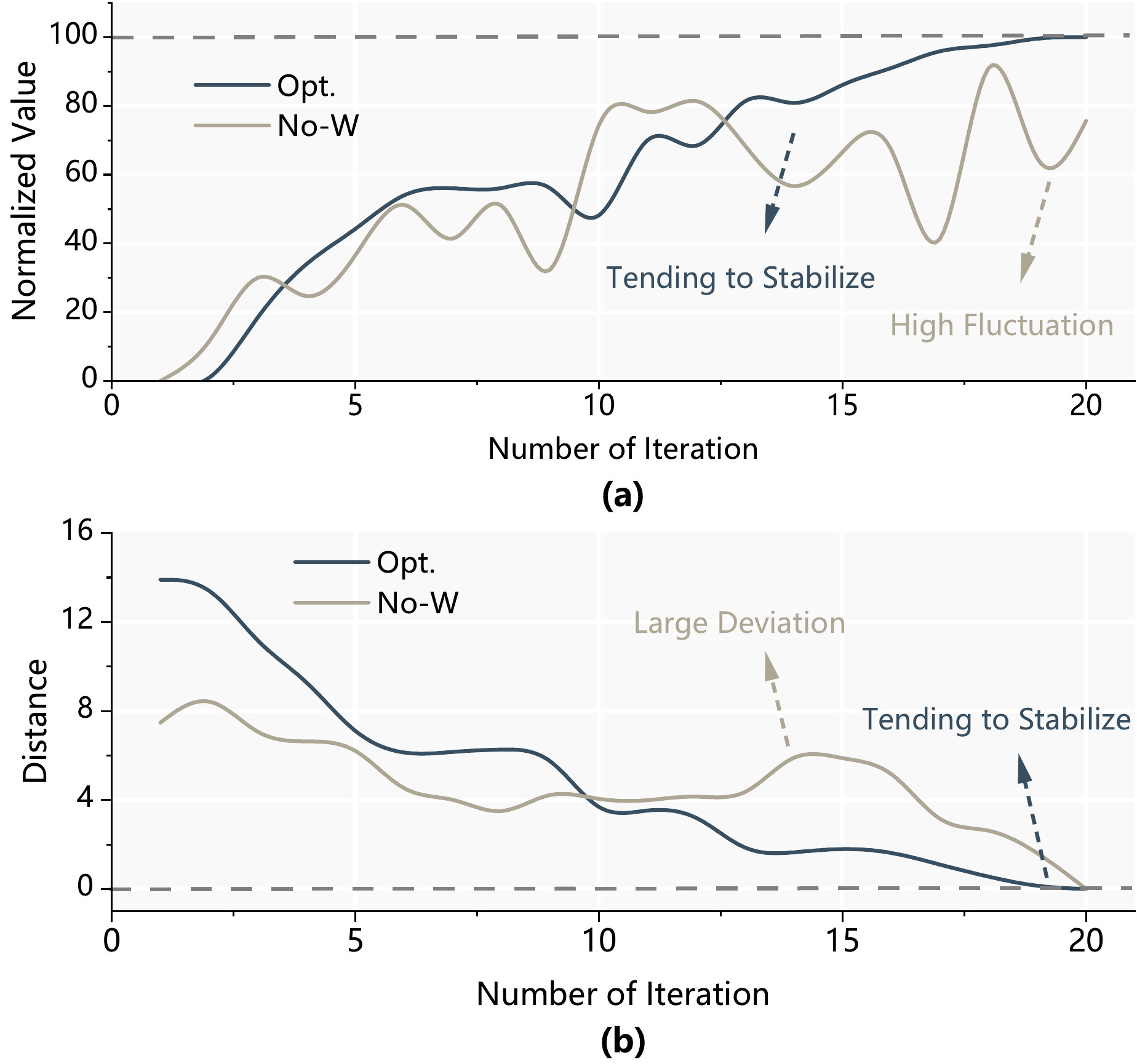}
	\caption{Analysis results of adaptive sampling distribution \(\omega_t\). (a) Curves of estimated mean values over $20$ iterations with and without \(\omega_t\). (b) Variations in distances between the six final optimized parameters and parameters at prior iterations for both schemes.}. 
	\label{fig_weight}
\end{figure}

\begin{figure}
	\centering
	\includegraphics[width=0.48\textwidth]{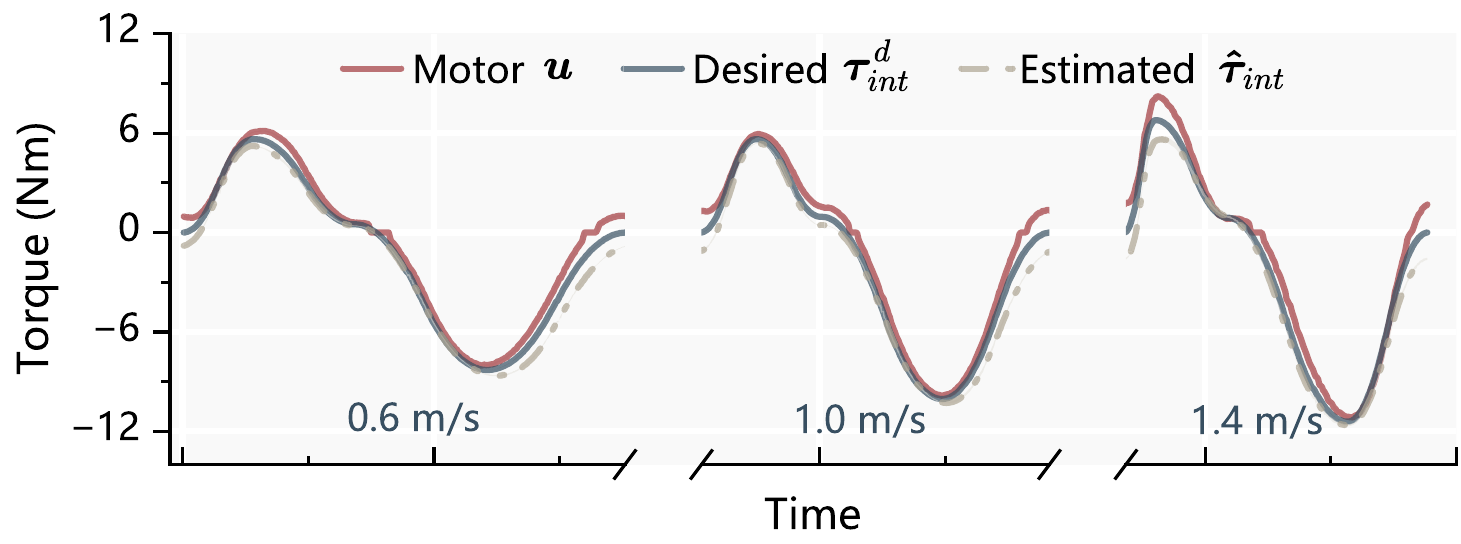}
	\caption{Comparison results among the motor output torque $\bb{u}$, the desired interaction torque $\bb{\tau}_{int}^d$, and the estimated interaction torque $\bb{\hat{\tau}}_{int}$.}. 
	\label{fig_control_curev}
\end{figure}

\subsubsection{\textbf{Effect of Adaptive Sampling Distribution}} Appendix C verifies that the introduction of the adaptive sampling distribution $\omega_t$ in the simulation environment significantly enhances the efficiency and performance of optimization. Fig. \ref{fig_weight} compares the actual results of the torque parameter optimization between the proposed algorithm and the baseline method without adaptive sampling $\omega_t$. Fig. \ref{fig_weight} (a) illustrates the comparison of the estimated mean objective values during the iterative optimization process. The results demonstrate that the proposed method (denoted as Opt.) achieves rapid convergence and stability within the limited $20$ iterations. In contrast, the method without adaptive sampling (denoted as No-W) shows an upward trend but a high fluctuation in the mean value, indicating that the algorithm has not yet stabilized. Fig. \ref{fig_weight} (b) illustrates the distances between the six optimized parameters at the $20$-th iteration and those obtained in all prior iterations. The proposed optimization strategy (Opt.) presents a decreasing trend and eventually converges steadily, which confirms its reliable optimization stability. By contrast, obvious fluctuations are observed without the adaptive sampling distribution, indicating inferior algorithm stability. These observations jointly verify that the introduction of an adaptive sampling distribution $\omega_t$ effectively improves the efficiency and stability of optimization.

\begin{figure}
	\centering
	\includegraphics[width=0.45\textwidth]{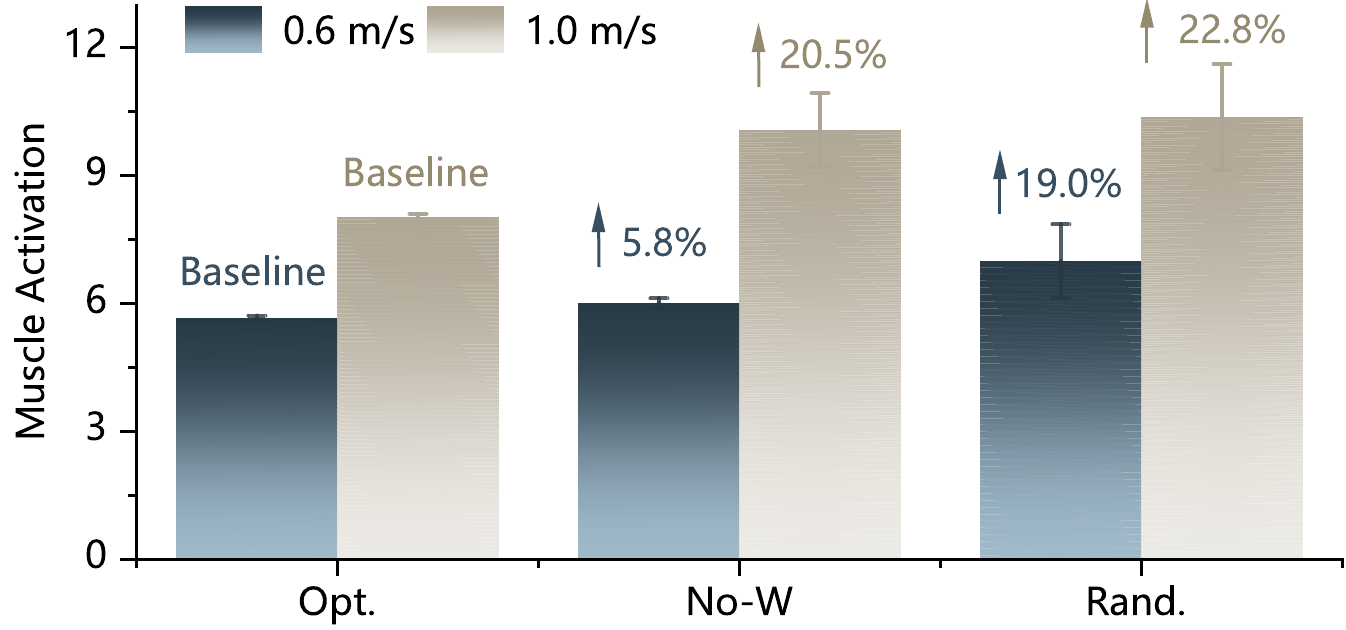}
	\caption{Comparison of peak muscle activation under different assistance torques at walking speeds $0.6$ and $1.0$ m/s. The abbreviation Opt. denotes the torque optimized by \texttt{PbBO}, Rand. denotes the randomly generated torque, and No-W denotes the optimized torque without the adaptive sampling distribution.}. 
	\label{fig_muscle_random}
\end{figure}

\subsection{Results of Assistance Performance} \label{Metabolic}
To evaluate whether the optimized torque profiles align with the user's locomotor preferences, this study comparatively analyzed physiological responses under three conditions: walking with exoskeleton assistance (Assist On), walking without assistance (Assist Off), and natural walking without the exoskeleton (No Exo). The evaluation metrics encompassed muscle activation levels, metabolic rate, and heart rate variations (The corresponding formulations are provided in Section \ref{experiments_setup}.).

\subsubsection{\textbf{Analysis of Motor Torque and Interaction Torque}} Given the inherent difficulties in obtaining direct measurements of human-robot interaction torque $\bb{\tau}_{int}$, this paper utilizes a nonlinear disturbance observer to estimate interaction torque. It controls motor output torque $\bb{u}$ to track the generated personalized interaction torque accurately. 
Fig. \ref{fig_control_curev} presents a comparative analysis between the estimated interaction torque $\hat{\bb{\tau}}_{int}$ (obtained via Eq. \eqref{eq_28}), the desired torque $\bb{\tau}_{int}^d$ (generated by the optimized parameters) and motor output torque $\bb{u}$ (obtained by Eq. \eqref{eq_control}) in the treadmill for different speeds. 
During deployment, the angles, angular velocities, and computed angular accelerations are all filtered with a Butterworth filter to mitigate noise.
Experimental results indicate that the RMSE error between the DOB-estimated interaction torque $\hat{\bb{\tau}}_{int}$ and the desired interaction torque $\bb{\tau}_{int}^d$ is $0.03$ N·m, demonstrating good tracking performance. The motor torque $\bb{u}$ is generally consistent with the individualized interaction torque $\bb{\tau}_{int}^d$, while notable deviation occurs at the peak positions.
This also indicates that the low-level controller satisfies the requirements of real-time assistance.

\begin{figure*}
  \centering
    \centering
    \includegraphics[width=\linewidth]{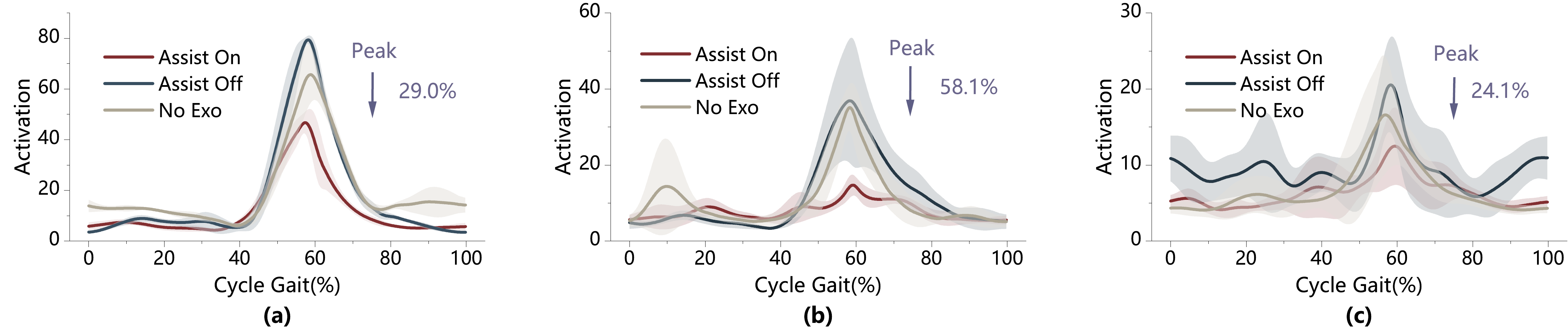}
  \caption{Comparison of average muscle excitation over a single gait cycle under three assistive conditions. (a), (b), and (c) report the experimental results obtained at walking speeds of 1.4 m/s, 1.0 m/s, and 0.6 m/s, respectively. For the muscle excitation results, the solid line denotes the mean value across multiple gait cycles, and the shaded region indicates the corresponding variance.}
  \label{fig_muscle}
\end{figure*}

\begin{figure*}
	\centering
	\includegraphics[width=0.98\textwidth]{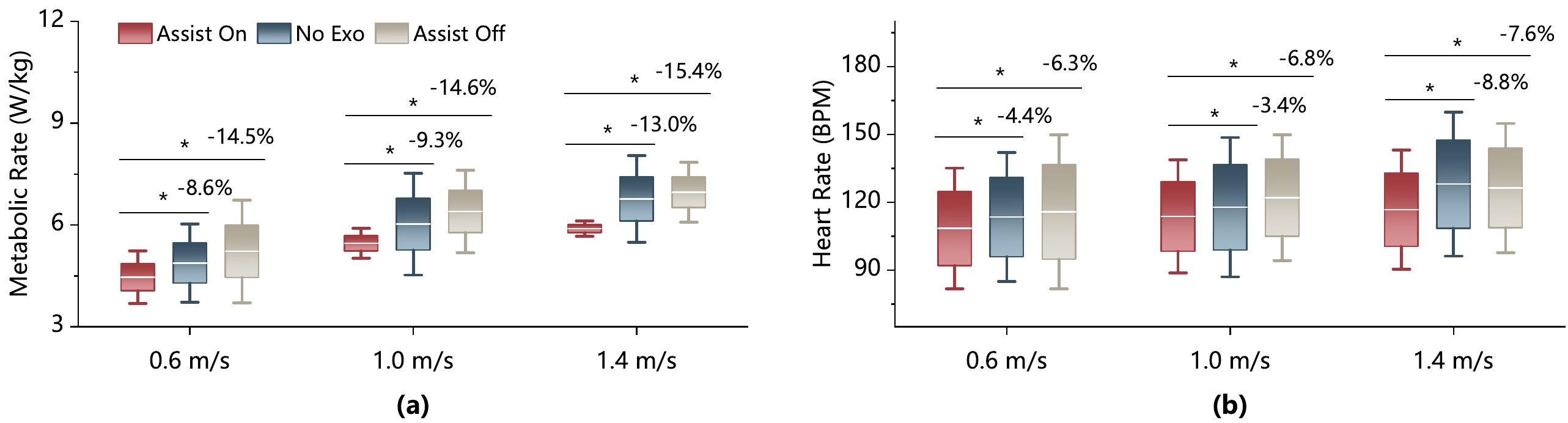}
	\caption{Comparison results of metabolic cost (a) and heart rate (b) under three assistance conditions (Assist On, No Exo and Assist Off) at three different treadmill speeds ($0.6$, $1.0$ and $1.4$ m/s).}. 
	\label{fig_metabolic_treadmill}
\end{figure*}

\subsubsection{\textbf{Muscle Activation Results}}
Fig. \ref{fig_muscle} presents a systematic comparison of muscle activation levels across three experimental conditions (Assist On, Assist Off, and No Exo) at three treadmill walking speeds, with all assistance tests conducted using optimized personalized assistance profiles generated by the framework proposed in this work. 
Compared with the No Exo and Assist Off control conditions, the user's thigh muscle activation levels are significantly reduced under the Assist On condition with optimized personalized assistance. 
Quantitatively, compared to the condition without wearing the exoskeleton (No Exo), the peak muscle excitation values after exoskeleton assistance at walking speeds of $1.4$, $1.0$, and $0.6$ m/s are reduced by $29.0\%$, $58.1\%$, and $24.1\%$, respectively. The integral of muscle excitation over one gait cycle is reduced by $31.5\%$, $26.7\%$, and $6.7\%$, respectively.
The above experimental results and analysis validate that the personalized assistance parameters optimized in the proposed framework can effectively assist the user by supplementing the torque output of the joint and reducing the muscular effort required for walking.

To further evaluate the efficacy of \texttt{PbBO}, Fig. \ref{fig_muscle_random} presents a comparative analysis of peak muscle activation levels, averaged over two subjects at two distinct walking speeds ($0.6$ and $1.0$ m/s). The comparison encompasses three torque assistance profiles: the optimized curve derived from \texttt{PbBO} (denoted as Opt.), a randomly generated curve (Rand.), and a curve obtained after $20$ iterations without an adaptive sampling distribution (No-W).
Note that each random parameter is constrained within a specified range, and the randomly generated torque curves still roughly conform to the general trend of the user's motion.
Experimental results indicate that, compared with the optimized torque, the curves generated with random parameters led to increases in average muscle activation of $19.0\%$ and $22.8\%$ under the two speed conditions, respectively. Without the curves optimized using the adaptive sampling distribution, muscle activation increased by $5.8\%$ and $20.5\%$, respectively.
This finding further validates that the adoption of an adaptive sampling distribution effectively enhances the optimization efficiency.

\subsubsection{\textbf{Metabolic Rate and Heart Rate Results}} Fig. \ref{fig_metabolic_treadmill} presents a comparative analysis of the average metabolic rate and heart rate during treadmill walking at speeds of $0.6$, $1.0$, and $1.4 \text{ m/s}$ under three experimental conditions: No Exo, Assist Off, and Assist On. 
Each trial comprises a walking session of $ 5$ minutes, with metabolic rates computed from breath-by-breath data collected during the final $3$ minutes to ensure steady-state measurements. 
The results demonstrate that the metabolic rates in the No Exo condition are $4.9 \pm 0.7$ (mean ± standard deviation), $6.0 \pm 1.2$, and $6.8 \pm 0.9 \text{ W/kg}$ under three tasks, respectively. In contrast, the Assist On condition yielded a significantly lower metabolic cost of $4.5\pm0.3$, $5.5\pm0.1$, and $5.9\pm0.0 \text{ W/kg}$, respectively.
Compared with the No Exo condition, exoskeleton assistance results in reductions in metabolic cost and heart rate of $8.6\%$, $9.3\%$, and $13.0\%$, and reductions in heart rate of $4.4\%$, $3.4\%$, and $8.8\%$, respectively. Compared with the Assist Off condition, the corresponding reductions are $14.5\%$, $14.6\%$, and $15.4\%$, and those in heart rate are $6.3\%$, $6.8\%$, and $7.6\%$, respectively.
The above results demonstrate that the optimized torque can effectively reduce the user's energy expenditure.

\begin{figure}
	\centering
	\includegraphics[width=0.40\textwidth]{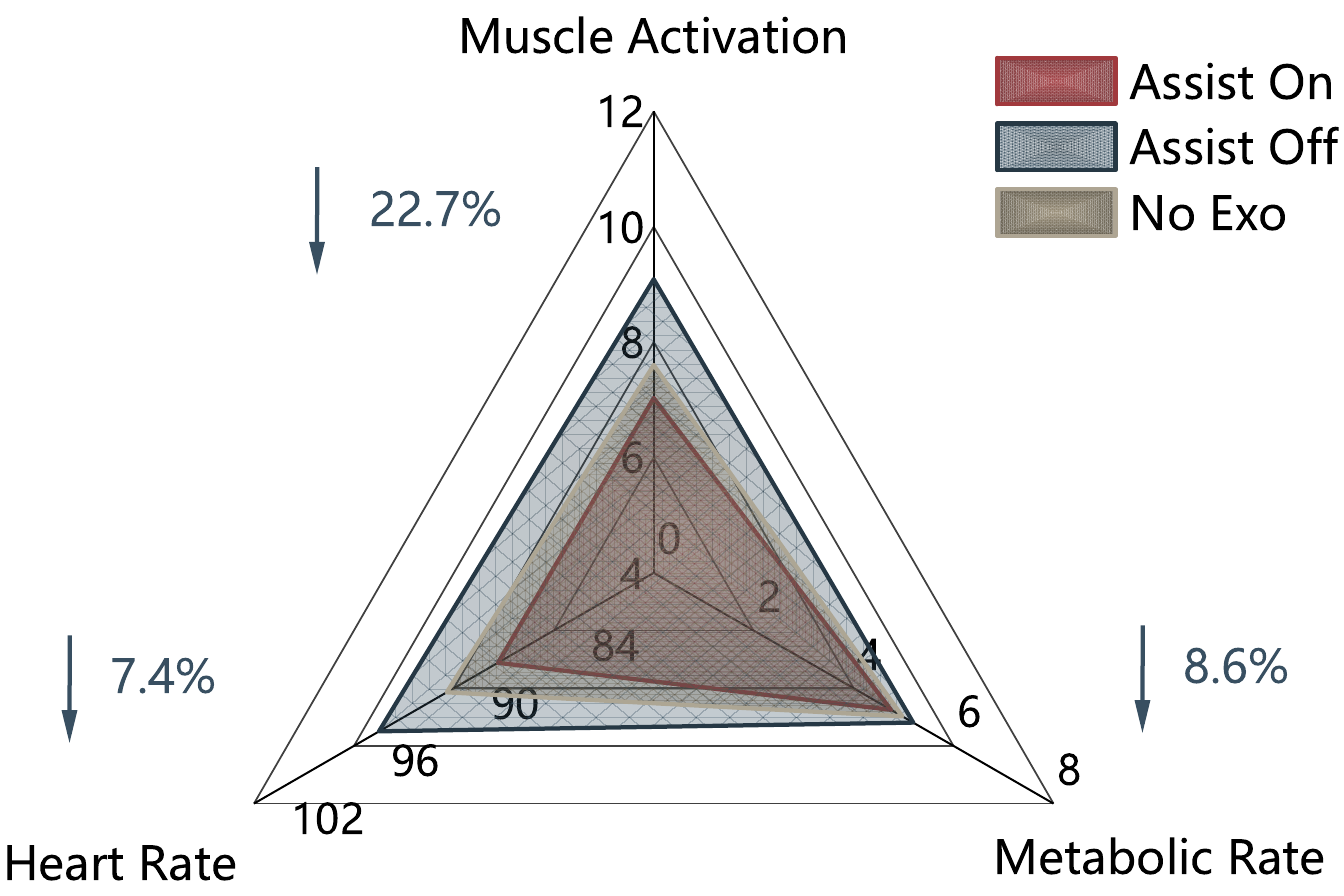}
	\caption{Experimental comparison of metabolic cost, heart rate, and average muscle activation under three conditions in outdoor assistance.}. 
	\label{fig_outdoor}
\end{figure}

Fig. \ref{fig_outdoor} compares the metabolic rate, heart rate, and average muscle activation during outdoor walking for three assistance conditions. Subjects are instructed to walk along the designated trajectory (Fig. \ref{fig_procedure} (b)), following the prescribed speed profile for each segment.
The total walking distance is $800$ m. Metabolic energy expenditure, heart rate, and muscle activation are synchronously recorded throughout the trial. The mean value of each variable over the entire course is adopted as the respective final metric.
Experimental results show that compared with the No Exo condition, exoskeleton assistance results in reductions in metabolic cost, heart rate, and muscle activation of $4.5\%$, $3.3\%$, and $7.5\%$, respectively. Compared with the Assist Off condition, the corresponding reductions are $8.6\%$, $7.4\%$, and $22.7\%$, respectively. The results demonstrate that the proposed optimization framework can effectively reduce user energy expenditure without requiring offline datasets or simulation environments.

\section{Discussion}\label{Discussion}
This section presents a thorough investigation into personalized assistance strategies for exoskeleton robots, focusing on two critical dimensions: the human-in-the-loop optimization (HILO) framework and joint torque estimation techniques. Then, we articulate the significance of the proposed method in the field of human–robot interaction, and critically examine the current limitations of the proposed system and outline promising avenues for future research and development.

\subsection{Discussion for Personalized Exoskeleton Assistance}
\subsubsection{\textbf{Human-in-the-loop Optimization}} Since human motion characteristics vary significantly across different individuals and tasks, personalized exoskeleton robots are of paramount importance. In current HILO approaches for exoskeletons, the algorithmic optimization objectives can be broadly categorized into three classes: individual impedance parameter optimization \cite{lizhi2022,2025TASE,10682783}, joint trajectory optimization \cite{2024chen,chen2024upper,sun2025optimization}, and torque curve optimization \cite{lee2023user,slade2022personalizing,arens2025preference}. The proposed framework in this work falls under the third category. The above aspects exhibit significant interdependencies: impedance control inherently outputs torque profiles, which fundamentally align with torque optimization objectives. Moreover, the impedance model itself relies on predicted joint trajectories, creating a coupled relationship between these optimization approaches. Li \e \cite{lizhi2022} regulated the impedance model with adaptive assistant powers for humans on different terrains. Chen \e \cite{2024chen,chen2024upper} optimized joint trajectories for different tasks and individuals, and then adjusted the impedance parameters through the generative model. The related works for torque optimization can be found in Section \ref{sec:Related work}. In summary, HILO algorithms require continuous human interaction for parameter optimization. While this inevitably introduces certain online computational costs, such approaches offer superior generalization capabilities and reduced dependence on extensive training datasets.

\subsubsection{\textbf{Human Joint Torque Estimation:}} Unlike human-in-the-loop algorithms, existing predictive joint torque algorithms rely heavily on large datasets. Current research in this field can be broadly categorized into two approaches: time-series prediction through multi-modal data fusion \cite{molinaro2024task,molinaro2024estimating,divekar2024versatile,liu2026exotraj}, and torque policy learning based on reinforcement learning \cite{luo2024experiment,2025TAI,luo2023robust}. The first approach primarily involves collecting multi-modal data from multiple subjects performing various tasks, followed by offline training of neural networks with temporal memory capabilities for online deployment \cite{molinaro2024task,molinaro2024estimating}. However, this method depends on professional motion capture systems, resulting in high data acquisition costs that limit its widespread adoption. Regarding the second approach, due to reinforcement learning's requirement for extensive interaction data and training instability, such methods typically rely on realistic simulation environments \cite{2025TAI}. These simulations model humans performing different walking tasks while employing reward functions to guide the agent in learning assistance policies \cite{luo2023robust,luo2024experiment}. Nevertheless, the reward functions quantifying assistance effectiveness often cannot be precisely expressed mathematically, which, to some extent, constrains performance improvement. In summary, while the aforementioned methods demonstrate promising research potential, critical challenges remain in addressing data scarcity, simulation-to-reality transfer, and model generalization capabilities.

\subsection{Practical Implications of \texttt{PbBO}}
In the field of human–robot interaction, certain individualized human dynamic data are inherently difficult to measure directly \cite{liu2026exotraj}. Although simulation platforms such as OpenSim can generate estimated data, they provide only partial information and their accuracy remains dependent upon the fidelity of the underlying musculoskeletal models \cite{molinaro2024estimating}. Moreover, significant inter-individual variability poses challenges for model generalization, making it difficult for a single model to accommodate all users.
Furthermore, human-related objective functions are difficult to articulate with precise mathematical formulations. Leveraging human feedback, such as preference information, enables the progressive inference of the underlying objective function, whose associated parameters can subsequently be optimized via suitable optimization algorithms \cite{lee2023user, arens2025preference}. Consequently, certain unmeasurable individualized dynamic data can be estimated.

However, prolonged online human–robot interaction tends to induce user fatigue, particularly among elderly individuals in daily life and patients in rehabilitation scenarios. To this end, the \texttt{PbBO} proposed in this work effectively enhances the efficiency of preference-based optimization, offering the following practical implications: 1) In human–robot interaction scenarios where dynamic data are difficult to measure, the \texttt{PbBO} algorithm can optimize the corresponding data curves through only a limited number of online human–robot interactions; 2) In medical rehabilitation scenarios, patient feedback and clinician evaluations can be integrated as preference input, and the \texttt{PbBO} algorithm can be employed to optimize the assistive torque profile, thereby maximizing the effectiveness of rehabilitation training.

\subsection{Discussion for Limitations and Future Work}
In Section \ref{sec:Related work}, a comparative analysis is presented between the proposed optimization framework and existing user preference optimization methods. Moreover, unlike conventional exoskeleton personalized assistance approaches, the proposed framework eliminates the dependency on extensive offline datasets or high-fidelity simulation environments. Instead, it determines optimal control parameters through minimal human-exoskeleton interaction, thereby substantially reducing user training overhead. However, the proposed methodology in this work has several limitations that warrant discussion:

1) The current system lacks environmental perception capabilities due to the absence of integrated sensors, which restricts experimental validation in complex terrains. Nevertheless, the proposed optimization framework remains theoretically applicable to such scenarios. Future research directions include incorporating neural network-based sensor fusion techniques for environmental awareness, thereby enabling effective assistance in complex environments.

2) The rigid structure of the exoskeleton may compromise user comfort during prolonged wear. Potential improvements involve the integration of elastic elements at joint mechanisms to enhance interaction compliance. Furthermore, the development of bio-inspired soft exoskeleton architectures could significantly improve wearing comfort and user acceptance.

3) The personalized torque generation relies on an offline optimization process. Although the proposed framework demonstrates superior efficiency in human-robot interaction cost reduction compared to existing preference-based methods, it still requires non-negligible human-in-the-loop time. Future work should investigate AI-driven optimization paradigms to establish more efficient exoskeleton assistance frameworks.

\section{Conclusion}\label{conclusion}
This paper combines preference learning and Bayesian optimization and proposes a novel user-preference-based optimization framework (\texttt{PbBO}). Optimization efficiency can be enhanced by introducing the adaptive sampling distribution of the candidate set. The proposed framework can determine relatively optimal control parameters with fewer human-robot interactions across different individuals and tasks. Moreover, the hierarchical control architecture is designed to achieve personalized assistance. The treadmill and real-world experimental results show that the metabolic rate, heart rate, and muscle activation can be effectively reduced through the torque optimized by \texttt{PbBO}.
This work highlights the potential of preference-based assistance optimization for hip exoskeletons, marking a step toward translating exoskeleton technologies into the real world.
	\bibliographystyle{IEEEtran}

	\newpage
\begin{appendix}\label{appendix}
\setcounter{equation}{0}
\renewcommand{\theequation}{A.\arabic{equation}}
\setcounter{figure}{0}
\renewcommand{\thefigure}{A.\arabic{figure}}
\setcounter{table}{0}
\renewcommand{\thetable}{A.\arabic{table}}

\subsection{Related Derivations} \label{app_der}

\subsubsection{\textbf{Derivations for Tracking Error}} According to Eq. \eqref{eq_26}- Eq. \eqref{eq_28}, we have \cite{mohammadi2013nonlinear,2017NDOB}
\begin{equation}
    \begin{aligned}
        \Delta{\dot{\bb{\tau}}} &= {\bb{\dot\tau}_{int}^d} - {\bb{\dot{\hat\tau}}_{int}}\\
        & = {\bb{\dot\tau}_{int}^d} +\bb{\dot{Z}} + \bb{\dot{P}}(\bb{q},\dot{\bb{q}})\\
        & = {\bb{\dot\tau}_{int}^d}+ \bb{L}\left(-\bb{Z}-\bb{u}+\bb{Q}-\bb{P}\right)+ \bb{\dot{P}}\\
        & = {\bb{\dot\tau}_{int}^d}+\bb{L}(\bb{P}+\bb{{\hat\tau}}_{int})-\bb{L}(\bb{u}-\bb{Q}+\bb{P})+ \bb{\dot{P}}\\
        & = {\bb{\dot\tau}_{int}^d}+\bb{L}\bb{{\hat\tau}}_{int}-\bb{L}(M\ddot{\bb{q}}+\bb{\tau}_{int}^d)+ \bb{\dot{P}}\\
        & = {\bb{\dot\tau}_{int}^d}-\bb{L}\Delta{{\bb{\tau}}}-\bb{L}M\ddot{\bb{q}}+ \bb{\dot{P}}.
    \end{aligned}
\end{equation}
Then, by Eq. \eqref{eq_28}, we have $-\bb{L}\Delta{{\bb{\tau}}}-\bb{L}M\ddot{\bb{q}}+ \bb{\dot{P}} = 0$. Therefore, we can conclude that
\begin{equation}
    \Delta{\dot{\bb{\tau}}} = -\bb{X}^{-1} \bb{M}^{-1}(\bb{q}) \Delta{{\bb{\tau}}}+\dot{\bb{\tau}}_{int}^d
\end{equation}
This completes the derivation for tracking error (Eq. \eqref{eq_track}).

\subsubsection{\textbf{Derivations for Preference-based GP Model}} 

This part shows the updating details of $\bb{\mu}_t$ in Eq. \eqref{eq_7} and $\mathbf{W}_t$ in Eq. \eqref{eq_9}. Here, the probability density function $\phi(x)$ and cumulative distribution function $\Phi(x)$ of the standard normal distribution are denoted as 
\begin{equation}
    \begin{aligned}
        \phi(x)=\frac{1}{\sqrt{2\pi}}e^{-x^2/2},\quad \Phi(x) = \int_{-\infty}^x \phi(t) dt.
    \end{aligned}
\end{equation}

To solve $\bb{\mu}_t$, we use the gradient descent method. The gradient $\nabla L(\bb{f})$ of Eq. \eqref{eq_8} can be derived as
\begin{equation}\label{eq_a.2}
    \nabla L(\bb{f}) =\frac{\partial L(\bb{f})}{\partial \bb{f}}= \bb{\Sigma}^{-1}_t\bb{f}+\bb{G}_t,
\end{equation}
where $\bb{\Sigma_t}$ can be calculated through Eq. \eqref{eq_4}, and $\bb{G} \in \mathcal{R}^{2t\times 2t}$ can be represented as
\begin{equation}
    G_{2i}=-\frac{p_i\cdot \phi(z_i)}{\sqrt{2}\sigma_n\cdot \Phi(z_i)},\quad i=0,1,...,t-1
\end{equation}
where $z_i=p_i[f(x_{i,0})-f(x_{i,1})]/(\sqrt{2}\sigma_n)$, and $G_{2i+1}=-G_{2i}$. Then, once $\bb{G}_t$ is obtained, and $\bb{\mu}_t$ can be optimized through Eq. \eqref{eq_a.2}. The posterior covariance can be solved through a second-order Taylor expansion under $\bb{\mu}_t$  \cite{chu2005preference}, that is
\begin{equation}
\begin{aligned}
         \nabla^2 L(\bb{f})&=\left.\frac{\partial^2 L(\bb{f})}{\partial f(\bb{x})\partial f^T(\bb{x})}\right|_{\bb{f}=\bb{\mu}_t}\\
         &=\bb{\Sigma}_t^{-1} +\frac{\partial^2 \sum_{i=1}^m -\log \Phi(z_i)}{\partial f(\bb{x})\partial f^T(\bb{x})}\\
         &=\bb{\Sigma}_t^{-1} +\mathbf{W}_t,
\end{aligned}
\end{equation}
Then, the $(i,j)$-th entry of $\mathbf{W}_t$ can be described as
\begin{equation}
    W_{i,j}=-\frac{\mathbb{I}_k(\bb{x}_i)\mathbb{I}_k(\bb{x}_j)p_k\phi(z_k)}{2\sigma_n^2}\left(\frac{\phi(z_k)+z_k\Phi(z_k)}{\Phi^2(z_k)}\right),
\end{equation}
where $\mathbb{I}_k(\bb{x})$ is a indicator function which is $1$ if $\bb{x}=\bb{x}_{k,0}$; $-1$ if $\bb{x}=\bb{x}_{k,1}$; otherwise 0. Therefore, $\mathbf{W}_t$ can be updated when $\bb{\mu}_t$ is obtained.

\subsubsection{\textbf{Derivations for Information Gain}} 
The information gain (mutual information) $I(\bb{y}_T|\bb{f}_T)$ can be expressed as the difference between information entropy, that is
\begin{equation}
    I(\bb{y}_T|\bb{f}_T)=H(\bb{y}_T)-H(\bb{y}_T|\bb{f}_T)
\end{equation}
where $H(\bb{y}_T)$ is prior entropy and $H(\bb{y}_T|\bb{f}_T)$ is conditional entropy. The $I(\bb{y}_T|\bb{f}_T)$ reflects the extent to which the observation $\bb{y}_T \in \mathcal{R}^{2T}$ reduces the uncertainty about $\bb{f}_T\in \mathcal{R}^{2T}$. For Gaussian distribution $X\sim\mathcal{N}(\mu,\Sigma)$, the information entropy $H(X)$ can be expressed as $H(X)=\frac{1}{2}\log |2\pi e\Sigma|$. Then, since \( \bb{y}_{T}=\bb{f}_T+\bb{n}_T \) and $\bb{n}_T \sim \mathcal{N}(0,\sigma_n^2\bb{I})$, the conditional entropy $H(\bb{y}_T|\bb{f}_T)$ can be derived as \cite{chowdhury2017kernelized}: 
\begin{equation}
    H(\bb{y}_T|\bb{f}_T) = \frac{1}{2}\log \left|2\pi e\sigma_n^2\bb{I}\right|=T\log \left(2\pi e\sigma_n^2\right).
\end{equation}

According to $\bb{y}_T \sim \mathcal{N}(0,\sigma_n^2\bb{I}+\bb{\Sigma}_T)$, we can derive:
\begin{equation}\label{eq_a.8}
    I(\bb{y}_T|\bb{f}_T)=\frac{1}{2}\log \Big|\bb{I}+\frac{\bb{\Sigma}_T}{\sigma_n^{2}}\Big|=\frac{1}{2}\sum_{i=1}^{K}\left(1+\frac{\lambda_i}{\sigma_n^{2}}\right).
\end{equation}
where $\lambda_1,\lambda_2,...,\lambda_K$ are the $K$ eigenvalues of $\bb{\Sigma}_T$. Then, we can conclude another type of $I(\bb{y}_T|\bb{f}_T)$. According to $P(f_{t}|\bb{f}_{t-1}, \bb{x}_{t})=\mathcal{N}(m_{t-1}(\bb{x}_{t}),c_{t-1}^2(\bb{x}_{t}))$, we can derive:
\begin{equation}
    H(y_t|\bb{y}_{t-1})=\frac{1}{2}\log \left[2\pi e\sigma_n^2\left(1+\frac{c_{t-1}^2(\bb{x}_{t})}{\sigma_n^{2}}\right)\right].
\end{equation}
Then, since $H(\bb{y}_T)=\sum_{t=1}^TH(y_t|\bb{y}_{t-1})$, the information gain (mutual information) $I(\bb{y}_T|\bb{f}_T)$ can be derived as
\begin{equation}
    \frac{1}{2}\sum_{t=1}^T\log \left(1+\frac{c^2_{t-1}(\bb{x}_{t,0})}{\sigma_n^{2}}\right)\left(1+\frac{c^2_{t-1}(\bb{x}_{t,1})}{\sigma_n^{2}}\right).
\end{equation}
This completes the derivation for information gain (Eq. \eqref{eq_17}).

\subsection{Related Proofs} \label{proof}
\subsubsection{\textbf{Proof for Theorem \ref{the_2}}} Since $\bb{K}_p$ and $\bb{K}_d$ are positive constant diagonal matrices, $\bb{A}$ is negative definite and diagonal (all eigenvalues of $\bb{A}$ are strictly negative). Therefore, we construct positive definite and diagonal matrix $\bb{P}$ which satisfies Lyapunov equation $\bb{A}^T\bb{P}+\bb{P}\bb{A}=-\bb{I}$. Furthermore, we define the Lyapunov candidate function:
\begin{equation}
    V = \bb{e}_r^T\bb{P}\bb{e}_r>0.
\end{equation}
By taking the derivative of the above equation, we have 
\begin{equation}
    \dot{V} = \dot{\bb{e}}_r^T\bb{P}\bb{e}_r+\bb{e}_r^T\bb{P}\dot{\bb{e}}_r.
\end{equation}
Substituting Eq. \eqref{eq_32} and $\bb{A}^T\bb{P}+\bb{P}\bb{A}=-\bb{I}$ into above equation, we can derive
\begin{equation}
     \dot{V} = -\bb{e}_r^T\bb{e}_r+2\bb{e}_r^T\bb{P}\bb{B}\Delta\bb{\tau}.
\end{equation}
Using the Cauchy-Schwarz inequality and $\|\Delta\bb{\tau}\|\leq C e^{-\alpha t}$, we can conclude
\begin{equation}
    \dot{V} \leq -\|\bb{e}_r\|^2+D\|\bb{e}_r\|e^{-\alpha t},
\end{equation}
where $D=2C\|\bb{P}\bb{B}\|$. When $t$ is sufficiently large, the term $e^{-\alpha t}$ diminishes in magnitude, ensuring $\dot{V}<0$ and thus resulting in the asymptotic convergence of $\bb{e}_r$ to zero.

In the following, we analyze the convergence rate of $\bb{e}_r$ to zero. We first introduce the auxiliary variable $W = \sqrt{V}$. Then, we have
\begin{equation}
    \dot{W}=\frac{\dot{V}}{2\sqrt{V}}\leq\frac{-\|\bb{e}_r\|^2+D\|\bb{e}_r\|e^{-\alpha t}}{2\sqrt{V}},
\end{equation}
Since $\lambda_{min}(P)\|\bb{e}_r\|^2\leq V \leq \lambda_{max}(P)\|\bb{e}_r\|^2$, we can derive that
\begin{equation}
    \dot{W}\leq -\frac{W}{2\lambda_{max}(P)} + \frac{De^{-\alpha t}}{2\sqrt{\lambda_{min}(P)}},
\end{equation}
Then, by solving the differential inequality, we derive the following equation
\begin{equation*}
    W(t)\leq \left(W(0)-\frac{G}{\xi-\alpha}\right)e^{-\xi t}+\frac{G}{\xi-\alpha}e^{-\alpha t}, \quad \xi \neq \alpha
\end{equation*}
where $G = D/2\sqrt{\lambda_{min}(P)}$ and $\xi = 1/(2\lambda_{max}(P))$. $W(t)$ converges to zero at a minimum exponential exponential convergence rate $\min \{\alpha,\xi\}$. 

Since $\sqrt{\lambda_{min}(P)}\|\bb{e}_r\|\leq W$, $\bb{e}_r$ also converges to zero at the same rate. $\bb{A}=-\bb{K}_d^{-1}(\bb{I}+\bb{K}_p)$ is negative definite and diagonal, and $\bb{P} = -\bb{A}^{-1}/2$ is only solution to equation $\bb{A}^T\bb{P}+\bb{P}\bb{A}=-\bb{I}$. Therefore, we derive that $\xi =\lambda_{min}(\bb{K}_d^{-1}(\bb{I}+\bb{K}_p))$. This completes the proof for Theorem \ref{the_2}.

\subsubsection{\textbf{Proof for Theorem \ref{the_1}}} In Lemma 5.5 of \cite{srinivas2010gaussian}, they derived that for any $\delta\in(0,1)$, set $\beta_t=2\log (\pi_t/\delta)$ and the below equation holds with probability at least $1-\delta$:
\begin{equation}\label{eq_a.10}
    \left|f(\bb{x}_t)-m_{t-1}(\bb{x}_t)\right|\leq \beta_t^{1/2} c_{t-1}(\bb{x}_t),
\end{equation}
where $\sum_t\pi_t^{-1}=1$. In our setting, $\beta_t=\gamma\beta_{t-1}=\gamma^t\beta_0$, thus $\beta_0$ can be derived as:
\begin{equation}\label{eq_a.11}
    \beta_0(T)=\frac{2(1-\gamma)}{\gamma-\gamma^{T+1}}\log \left(\frac{\prod_{t=1}^{T} \pi_t}{\delta^T}\right),
\end{equation}
where $\gamma$ is decay factor and $\gamma\in(0,1)$. At each round $t$, the $\bb{x}_t$ selected by Eq. \eqref{eq_12}, we have $m_{t-1}(\bb{x}^*)+{\beta_{t}}^{1/2}c_{t-1}(\bb{x}^*)\leq m_{t-1}(\bb{x}_t)+{\beta_{t}}^{1/2}c_{t-1}(\bb{x}_t)$. Then, according to Eqs. \eqref{eq_19} and \eqref{eq_a.10}, we can derive 
\begin{equation}
    \begin{aligned}
        r_t=&f(\bb{x}^*)-f(\bb{x}_{t,0})+f(\bb{x}^*)-f(\bb{x}_{t,1})\\
        \leq & m_{t-1}(\bb{x}_{t,0})-f(\bb{x}_{t,0})+{\beta_{t-1,0}}^{1/2}c_{t-1}(\bb{x}_{t,0})+\\
        &m_{t-1}(\bb{x}_{t,1})-f(\bb{x}_{t,1})+{\beta_{t-1,1}}^{1/2}c_{t-1}(\bb{x}_{t,1})\\
        \leq &2{\beta_{t-1,0}}^{1/2}c_{t-1}(\bb{x}_{t,0})+2{\beta_{t-1,1}}^{1/2}c_{t-1}(\bb{x}_{t,1})\\
        \leq &2\sqrt{\beta_0(T)}\left(c_{t-1}(\bb{x}_{t,0})+c_{t-1}(\bb{x}_{t,1})\right).
    \end{aligned}
\end{equation}
where $\beta_{0}(T)=\max\{\beta_{0,0},\beta_{1,0}\}$, the last inequality holds because $\beta_t=\gamma \beta_{t-1}$ and $\gamma<1$. $\beta_{0,0}$ and $\beta_{1,0}$ are the initial value determined through Eq. \eqref{eq_a.11}, which are used to solve $\bb{x}_{t,0}$ and $\bb{x}_{t,1}$, respectively. Then, the cumulative regret bound
$R_T\leq 2\sqrt{\beta_0(T)}\sum_{t=1}^T \left(c_{t-1}(\bb{x}_{t,0})+c_{t-1}(\bb{x}_{t,1})\right)$. By the Cauchy-Schwarz inequality, we have
\begin{equation}
\begin{aligned}
        &\sum_{t=1}^T \Big(c_{t-1}(\bb{x}_{t,0})+c_{t-1}(\bb{x}_{t,1})\Big)\\
        \leq&\sqrt{2T\sum_{t=1}^T\Big(c_{t-1}^2(\bb{x}_{t,0})+c_{t-1}^2(\bb{x}_{t,1})\Big)},
\end{aligned}
\end{equation}

\begin{figure}
	\centering
	\includegraphics[width=0.38\textwidth]{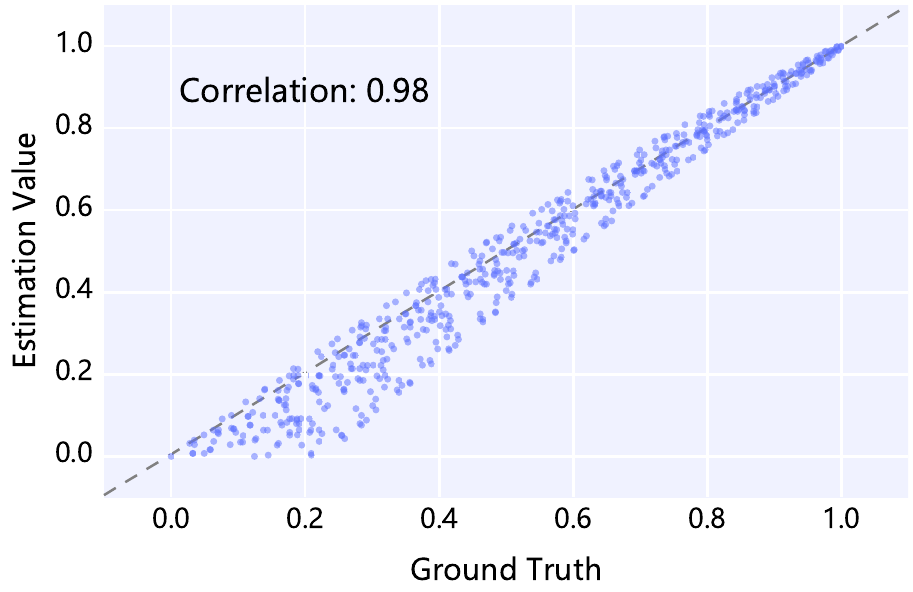}
	\caption{The comparison between the prediction value by the learned preference model and the ground truths of function $f(\bb{x})$. The value estimated by the learned preference model and the ground-truth value are both normalized to $[0,1]$.}. 
	\label{fig_preference}
\end{figure}

Considering that inequality $\log(1+x)>x/2$ holds when $0<x<1$, and $0<c_{t}^2<1$ and $0<\sigma_n^2<1$, the information gain $I(\bb{y}_T|\bb{f}_T)$ satisfies:
\begin{equation}
    I(\bb{y}_T|\bb{f}_T)\geq \frac{1}{4}\sum_{t=1}^T \Big(c^2_{t-1}(\bb{x}_{t,0})+c^2_{t-1}(\bb{x}_{t,1})\Big)
\end{equation}
Then, according to Eq. \eqref{eq_18}, $\mathcal{H}_T\geq I(\bb{y}_T|\bb{f}_T)$. Therefore, the cumulative regret bound satisfies:
\begin{equation}
    R_T\leq 4\sqrt{2\beta_0(T)T\mathcal{H}_T}=\mathcal{O}\Big(\sqrt{\beta_0(T)T\mathcal{H}_T}\Big).
\end{equation}
This completes the proof for Theorem \ref{the_1}.

\subsection{Toy Example for the Proposed Optimization Framework} \label{toy_experiments}
In Section \ref{method}, we illustrate how to combine preference learning and Bayesian Optimization, and improve optimization efficiency and stability through introducing an adaptive sampling distribution $\omega_t$ and the new form of trade-off factor $\beta_t$. Since the true objective function of the exoskeleton is unknown, it is challenging to visualize the discrepancy between the learned preferences (implicit objective function) and the true objective function, as well as the corresponding training process curves. To facilitate the visualization of the results of the proposed preference-based optimization, we define a simple optimization goal $f(\bb{x})$ here, which has a unique optimal solution:
\begin{equation}\label{eq_e.1}
    f(\bb{x})=-(x_1+2)^2-(x_2-2)^2-x_3^2-(x_4-1)^2,
\end{equation}
where $\bb{x}\in [-10,10]^4$. Obviously, the $f(\bb{x})$ can be achieved maximin value at $[-2,2,0,1]$. In the next part, we analyze the performance of the learned preference model and compare the optimization efficiency between our proposed method and the typical Bayesian Optimization method, GP-UCB (no sample distribution) \cite{srinivas2010gaussian}. Please note that the preference model component remains unchanged; only the optimization methodology has been modified. In this toy example, $\beta_0$ is set as $0.25$, $\gamma$ is set as $0.99$, the noise level $\sigma_n$ is $0.1$, and hyper-parameter $\theta$ in Eq. \eqref{eq_4} is $0.1$. The preference labels are given according to the $f(\bb{x})$ values corresponding to the two sets of points.

\begin{figure*}
	\centering
	\includegraphics[width=1\textwidth]{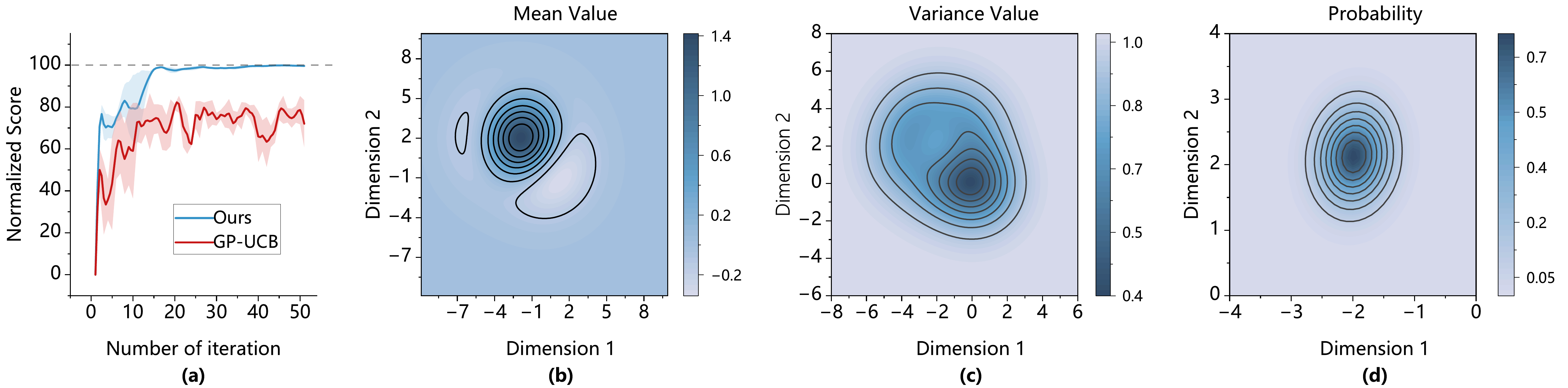}
	\caption{The visualization results of the proposed method. (a) compares the optimization curve between our improved BO and standard BO (GP-UCB), where the true score is computed through $f(\bb{x})$, and the scores are normalized into $[0,100]$ and the reference minimal and maximal values are $-5$ and $0$, respectively. (b) and (c) show the estimated mean and variance value distribution in variable space. Since the space is four-dimensional, for the convenience of visualization, we fix the third and fourth dimensions, and the values are set as $0$ and $1$, respectively. (d) shows the sampling probability of the designed sampling distribution $\omega$ in the whole space, where we also fix the last two dimensions.}. 
	\label{fig_toy_preference}
\end{figure*}

\textbf{The accuracy of the preference model based on GP.} Since the actual (true) optimization function is known to us, the evaluation value of the preference model at specific points can be compared with the true value. The value $\mu$ of the preference model can be obtained through Eq. \eqref{eq_7}. To ensure the consistency of comparison, the true value and the estimated value through preference learning are normalized to $[0,1]$, which can reflect the relative accuracy of the learned preference model. Fig. \ref{fig_preference} compares the results between the estimated value through the preference model and the true value for the objective function $f(\bb{x})$. This figure indicates that the learned preference model based on GP is basically consistent with the true objective function, and the correlation reaches $0.98$. In summary, the GP-based preference model can accurately reflect human preferences (in this case, the relative size of the $f(\bb{x})$ value).

\textbf{The optimization efficiency of the proposed optimization method.} Fig. \ref{fig_toy_preference} illustrates the performance of the proposed optimization framework. Fig. \ref{fig_toy_preference} (a) compares the optimization performance and speed between our improved BO method and the standard method (GP-UCB). It shows that the optimization speed, performance, and stability can be significantly improved through modifying the trade-off factor $\beta$ and introducing an adaptive sampling distribution $\omega$. Fig. \ref{fig_toy_preference} (b) and Fig. \ref{fig_toy_preference} (c) present the distribution of mean and variance values in the whole space. For the convenience of visualization, we fix the last two dimensions ($x_3=0,x_4=1$) and give the distribution of estimated mean and variance values about the first two dimensions. They indicate that the mean value is relatively large and the variance value is relatively small in the region centered at point $(-2,2)$. Fig. \ref{fig_toy_preference} (d) shows the sampling probability of the points under the designed adaptive sampling distribution. It indicates that in the region centered on point $(-2,2)$, the sampling probability is the highest, which is conducive to increasing the optimization speed. Therefore, the optimization speed, performance, and stability of our proposed optimization framework are better than those of the standard BO algorithm.

\end{appendix}

\end{document}